\RequirePackage{fix-cm}
 \pdfoutput=1
\documentclass[smallcondensed]{svjour3}     % onecolumn (ditto)
\smartqed  % flush right qed marks, e.g. at end of proof
\usepackage{graphicx}
\usepackage{amssymb}
\usepackage{amsmath}
\usepackage{mathtools}
\usepackage{nicefrac}
\usepackage{subfig}
\usepackage[labelfont=bf, justification=justified]{caption}
\usepackage[title]{appendix}
\usepackage{bm} %Bold notation for matrices and vectors
\DeclareMathOperator{\KL}{KL}
\DeclareMathOperator{\TV}{TV}

\newcommand{\myvec}[1]{\bm{\mathit{#1}}}

\newcommand{\diff}{\mathop{} \! \mathrm{d}}
 \journalname{Vietnam Journal of Mathematics}
\begin{document}

\title{A Survey on Surrogate Approaches to Non-negative Matrix Factorization%\thanks{Grants or other notes
}


\author{Pascal Fernsel         \and
        Peter Maass %etc.
}


\institute{P. Fernsel \at
              Center for Industrial Mathematics, University
              of Bremen, D-28334 Bremen, Germany \\
              Tel.: +49-421-218-63814\\
              \email{pfernsel@math.uni-bremen.de}           %  \\
           \and
           P. Maass \at
              Center for Industrial Mathematics, University
              of Bremen, D-28334 Bremen, Germany \\
              Tel.: +49-421-218-63 801 \\
              Fax.: 49-421-218-98 63 801 \\
              \email{pmaass@math.uni-bremen.de}
}

\date{}

\maketitle

\begin{abstract}
Motivated by applications in hyperspectral imaging we investigate methods for approximating a high-dimensional non-negative matrix $\myvec{Y}$ by a product of two lower-dimensional, non-negative matrices $\myvec{K}$ and $\myvec{X}.$ This so-called non-negative matrix factorization is based on defining suitable Tikhonov functionals, which combine a discrepancy measure for $\myvec{Y}\approx\myvec{KX}$ with penalty terms for enforcing additional properties of $\myvec{K}$ and $\myvec{X}$.\newline
The minimization is based on alternating minimization with respect to $\myvec{K}$ or $\myvec{X}$, where in each iteration step one replaces the original Tikhonov functional by a locally defined surrogate functional. The choice of surrogate functionals is crucial: It should allow a comparatively simple minimization and simultaneously its first order optimality condition should lead to multiplicative update rules, which automatically preserve non-negativity of the iterates.\newline
We review the most standard construction principles for surrogate functionals for Frobenius-norm and Kullback-Leibler discrepancy measures. We extend the known surrogate constructions by a general framework, which allows to add a large variety of penalty terms.\newline
The paper finishes by deriving the corresponding alternating minimization schemes explicitely and by applying these methods to MALDI imaging data.
\end{abstract}

\section{Introduction} \label{sec:Introduction}
Matrix factorization methods for large scale data sets have seen increasing scientific interest recently due to their central role for a large variety of machine learning tasks. The main aim of such approaches is to obtain a low-rank approximation of a typically large data matrix by factorizing it into two smaller matrices.  One of the most widely used matrix factorization method is the principal component analysis (PCA), which uses the singular value decomposition (SVD) of the given data matrix. \newline
In this work, we review the particular case of non-negative matrix factorization (NMF), which is favorable for a range of applications where the data under investigation naturally satisfies a non-negativity constraint. These include dimension reduction, data compression, basis learning, feature extraction as well as higher level tasks such as classification or clustering \cite{demol,leuschner18,LD13,Phon-Amnuaisuk13}. PCA based approaches without any non-negativity constraints would not lead to satisfactory results in this case since possible negative entries of the computed matrices cannot be easily interpreted for naturally non-negative datasets.\newline
Typically, the NMF problem is formulated as a minimization problem. The corresponding cost function includes a suitable discrepancy term, which measures the difference between the data matrix and the calculated factorization, as well as penalty terms to tackle the non-uniqueness of the NMF, to deal with numerical instabilities but also to provide the matrices with desirable properties depending on the application task. The NMF cost functions are commonly non-convex and require tailored minimization techniques to ensure the minimization but also the non-negativity of the matrix iterates. This leads us to the so called surrogate minimization approaches, which are also known as majorize-minimization algorithms \cite{hunter,lange,zhang}. Such surrogate methods have been investigated intensively for some of the most interesting discrepancy measures and penalty terms \cite{demol,defrise,fevotte,lange,leeseung,TF12,zhang}. The idea is to replace the original cost function by a so called surrogate functional, such that its minimization induces a monotonic decrease of the objective function. It should be constructed in such a way that it is easier to minimize and that the deduced update rules should preserve the non-negativity of the iterates, which typically leads to alternating, multiplicative update rules.\newline
It appears, that these constructions are obtained case-by-case employing different analytical approaches and different motivations for their derivation. The purpose of this paper, first of all, is to give a unified approach to surrogate constructions for NMF discrepancy functionals. This general construction principle is then applied to a wide class of functionals obtained by different combinations of divergence measures and penalty terms, thus extending the present state of the art for surrogate based NMF-constructions.\newline 
Secondly, one needs to develop minimization schemes for these functionals. Here we develop concepts for obtaining multiplicative minimization schemes, which automatically preserve non-negativity without the need for further projections.\newline
Finally, we exemplify some characteristic properties of the different functionals with MALDI imaging data, which are particularly high-dimensional and challenging hyperspectral data sets.\newline
The paper is organized as follows. Section \ref{sec:NMF} introduces the basic definition of the considered NMF problems. Section \ref{sec:Surrogate_functionals} gives an overview about the theory of surrogate functionals as well as the construction principles. This is then exemplified in Section \ref{sec:surrogates_for_NMF_functionals} for the most important cases of discrepancy terms, namely the Frobenius norm and the Kullback-Leibler divergence as well as for a variety of penalty terms. Section \ref{sec:Surrogate_based_NMF_algorithms} discusses alternating minimization schemes for these general functionals with the aim to obatin non-negativity-preserving, multiplicative iterations. Finally, Section \ref{sec:MALDI_Imaging} contains numerical results for MALDI imaging data.

\subsection{Notation} \label{subsec:notation}
Throughout this work, we will denote matrices in bold capital Latin or Greek letters (e.g. $ \myvec{Y, K, \Psi, \Lambda} $) while vectors will be written in small bold Latin or Greek letters (e.g. $ \myvec{c, d, \beta, \zeta} $). The entries of matrices and vectors will be indicated in a non-bold format to distinguish between the $ i $-th entry $ x_i $ of a vector $ \myvec{x} $ and $ n $ different vectors $ \myvec{x}_j $ for $ j=1,\dots,n. $ In doing so, we write for the entry of a matrix $ \myvec{M} $ in the $ i $-th row and the $ j $-th column $ M_{ij} $ and the $ i $-th entry of a vector $ \myvec{x} $ the symbol $ x_i. $ The same holds for an entry of a matrix product: the $ ij $-th entry of the matrix product $ \myvec{MN} $ will be indicated as $ (MN)_{ij}.$\newline 
Furthermore, we will use a dot notation to indicate rows and columns of matrices. For a matrix $ \myvec{M} $ we will write $ \myvec{M}_{\bullet, j} $ for the $ j $-th column and $ \myvec{M}_{i, \bullet} $ for the $ i $-th row of the matrix.\newline
What is more, we will use $ \Vert \cdot \Vert $ for the usual Euclidean norm, $ \Vert \myvec{M}\Vert_1\coloneqq \sum_{ij} \vert M_{ij}\vert $ for the 1-norm and $ \Vert \myvec{M}\Vert_F$ for the Frobenius norm of a matrix $ \myvec{M}. $\newline
Besides that, we will use equivalently the terms \textit{function} and \textit{functional} for a mapping into the real numbers.\newline
Finally, the dimensions of the matrices in the considered NMF problem are reused in this work and will be introduced in the following section.
\section{Non-negative Matrix Factorization} \label{sec:NMF}
Before we introduce the basic NMF problem, we give the following definition to clarify the meaning of a non-negative matrix.
\begin{definition}
	A matrix $ \myvec{M}\in \mathbb{R}^{m\times n} $ is called \textbf{non-negative} if $ \myvec{M}\in \mathbb{R}_{\geq0}^{m\times n}, $ where $ \mathbb{R}_{\geq 0}\coloneqq \{ x\in \mathbb{R}  :  x\geq 0 \}. $
\end{definition}
The non-negativity of an arbitrary matrix $ \myvec{M} $ will be abbreviated for simplicity as $ \myvec{M} \geq 0 $ in the later sections of this work.\newline
The basic NMF problem requires to approximately decompose a given non-negative matrix $\myvec{Y}\in \mathbb{R}_{\geq 0}^{n\times m}$ into two smaller non-negative matrix factors  $\myvec{K}\in \mathbb{R}_{\geq 0}^{n\times p}$ and $\myvec{X}\in \mathbb{R}_{\geq 0}^{p\times m},$ such that $p \ll \min(n,m)$ and
\begin{equation*}
\myvec{Y}\approx \myvec{KX}.
\end{equation*}
For an interpretation let us assume, that we are given $m$ data vectors $\myvec{Y}_{\bullet, j} \in \mathbb{R}^n$ for $ j=1, \dots, m, $ which are stored column-wise in the matrix $\myvec{Y}.$ Similarly for $k=1,\dots,p$ we denote by $\myvec{K}_{\bullet, k}$, respectively $\myvec{X}_{k, \bullet}$, the column vectors of $\myvec{K}$, respectively the row vectors of $\myvec{X}.$ We then obtain the following approximation for the column vectors $\myvec{Y}_{\bullet, j} $ as well as the row vectors $\myvec{Y}_{i,\bullet}: $
\begin{align*}
\myvec{Y}_{\bullet, j} &\approx \sum_{k=1}^{p} \myvec{K}_{\bullet, k} X_{kj}, \\
\myvec{Y}_{i,\bullet} &\approx \sum_{k=1}^{p} K_{ik}\myvec{X}_{k,\bullet},\\
\myvec{Y}\approx \myvec{KX} &= \sum_{k=1}^{p} \myvec{K}_{\bullet, k} \myvec{X}_{k, \bullet}.
\end{align*}
Note that the product $ \myvec{K}_{\bullet, k} \myvec{X}_{k, \bullet} $ on the right hand side of the third equation yields rank-one matrices for every $ k. $\newline
By these representations, we can regard the rows $\myvec{X}_{k,\bullet}$ as a low-dimensional set of basis vectors, which are tailored for approximating the high-dimensional data vectors, i.e. NMF solves the task of basis learning with non-negativity constraints.\newline
Following the interpretation given above, we can also regard NMF as a basis for compression. $\myvec{K}$ and $\myvec{X}$ are determined by storing $(n+m) \cdot p$ coefficients, as opposed to $n\cdot m$ coefficients for $\myvec{Y}$. The columns of $\myvec{K}$ can be regarded as characteristic components of the given data set $\{\myvec{Y}_{\bullet, j}\}_j$. If these data vectors are input for a classification task, one can use the $p$ correlation values with the column vectors of $\myvec{K}$ as features for constructing the classification scheme, which yields efficient and qualitatively excellent classifications, see \cite{leuschner18,Phon-Amnuaisuk13,TCP11}.\newline
The standard variational approach for constructing an NMF is to define a suitable discrepancy measure $D(\cdot,\cdot)$ between $\myvec{Y} $ and $ \myvec{KX} $ and to minimize the resulting functional. Despite their seemingly simple structure NMF problems are ill-posed, non-linear and non-convex, i.e. they require stabilization techniques as well as tailored approaches for minimization. In this paper, we consider discrepancy measures based on divergences \cite{hennequin}.

\begin{definition}[Divergence]
	Let $ \Omega $ be an arbitrary set. A \textbf{divergence} $ D $ is a map $ D: \Omega \times \Omega \to \mathbb{R}, $ which fulfills the following properties:
	\begin{itemize}
		\item[(i)] $ D(x, y) \geq 0 \quad \forall (x, y)\in \Omega \times \Omega $ \label{itm:def:Divergenz:1}
		\item[(ii)] $ D(x, y) = 0 \Leftrightarrow x=y$ \label{itm:def:Divergenz:2}
	\end{itemize}
\end{definition}

\begin{definition}[$ \beta $-divergence] \label{def:beta-divergence}
	The $ \beta $\textbf{-divergence} $ d_\beta : \mathbb{R}_{>0} \times \mathbb{R}_{>0} \rightarrow \mathbb{R}_{\geq 0} $ for $ \beta \in \mathbb{R} $ is defined as
	\begin{align}\label{eq:def:beta-Divergenz:Skalar}
	d_\beta (x, y) \coloneqq \begin{cases}
	\dfrac{x^\beta}{\beta (\beta-1)} + \dfrac{y^\beta}{\beta} - \dfrac{xy^{\beta-1}}{\beta - 1} \quad &\text{for} \ \beta\in \mathbb{R}\setminus \{0, 1\}, \\
	x \log \left (\dfrac{x}{y} \right ) - x + y \quad &\text{for} \ \beta=1, \\
	\dfrac{x}{y} - \log \left (\dfrac{x}{y}\right ) - 1 \quad &\text{for} \ \beta=0.
	\end{cases}
	\end{align}
	Furthermore, we define accordingly $ D_\beta: \mathbb{R}_{>0}^{n\times m} \times \mathbb{R}_{>0}^{n\times m} \rightarrow \mathbb{R}$ for arbitrary $ m, n\in \mathbb{N} $ as
	\begin{equation}\label{eq:def:beta-Divergenz:Matrix}
	D_\beta (\myvec{M}, \myvec{N}) = \sum_{i=1}^{n} \sum_{j=1}^{m} d_\beta(M_{ij}, N_{ij}).
	\end{equation}
\end{definition}
The corresponding matrix divergences are defined componentwise, i.e. $\beta =2$ yields the Frobenius norm and $\beta=1$ the Kullback-Leibler divergence.\newline
These discrepancy measures are typically amended by so called penalty terms for stabilization and for enforcing additional properties such as sparsity or orthogonality.
This yields the following general minimization task.

\begin{definition}[NMF Minimization Problem] \label{def:NMF-Minimumproblem}
	For a data matrix $ \myvec{Y}\in \mathbb{R}_{\geq 0}^{n\times m}, $ we consider the following generalized NMF minimization task
	\begin{equation}\label{eq:def:NMF-Minimumproblem:Minimumproblem}
	\min_{\myvec{K} \geq 0, \myvec{X}\geq 0 } D_\beta (\myvec{Y}, \myvec{KX}) + \sum_{\ell=1}^{L} \alpha_{\ell} \varphi_\ell (\myvec{K}, \myvec{X}).
	\end{equation}
	The functional
	\begin{equation}\label{eq:def:NMF-Minimumproblem:Kostenfunktional}
	F(\myvec{K}, \myvec{X}) \coloneqq D_\beta (\myvec{Y}, \myvec{KX}) + \sum_{\ell=1}^{L} \alpha_{\ell} \varphi_\ell (\myvec{K}, \myvec{X})
	\end{equation}
	is called the \textbf{cost functional}. Furthermore, we call
	\begin{itemize}
		\item[(i)] $ D_\beta (\myvec{Y}, \myvec{KX}) $ the \textbf{discrepancy term}, \label{itm:def:NMF-Minimumproblem:Diskrepanzterm}
		\item[(ii)] $ \alpha_{\ell} $ the \textbf{regularization parameters} or \textbf{weights} \label{itm:def:NMF-Minimumproblem:Gewichtungen}
		\item[(iii)] and $ \varphi_\ell (\myvec{K}, \myvec{X}) $ the \textbf{penalty terms}. \label{itm:def:NMF-Minimumproblem:Strafterme}
	\end{itemize}
\end{definition}
%Divergence measures are typically non-convex. To see this, we consider the simple case $n=m=1$ and $\beta = 2$. All entities are scalar in this case with $d_2(1,3) = d_2(3,1)=4$. But the value at the midpoint between $(1,3)$ and $(3,1)$, i.e. $d_2(2,2)=9$ is larger implying that $d_2$ is not convex.\newline
The functional in \eqref{eq:def:NMF-Minimumproblem:Kostenfunktional} is typically non-convex in $(\myvec{K},\myvec{X})$. Hence, algortihms based on alternating minimization with respect to $\myvec{K}$ or $\myvec{X}$ are favourable, i.e.

\begin{align}
\myvec{K}^{[d+1]} &= \arg\min_{\myvec{K}\geq 0}  F(\myvec{K}, \myvec{X}^{[d]}), \label{eq:AlternatingMinimzation1} \\
\myvec{X}^{[d+1]} &= \arg\min_{\myvec{X}\geq 0}  F(\myvec{K}^{[d+1]}, \myvec{X}), \label{eq:AlternatingMinimzation2}
\end{align}
where the index $ d $ denotes the iteration index of the corresponding matrices.\newline
This yields simpler, often convex  restricted problems with respect to either $\myvec{K}$ or $\myvec{X}$. Considering for example the minimization of  the NMF functional with Frobenius norm and without any penalty term, yields a high dimensional linear system $\myvec{K}^\intercal \myvec{Y} = \myvec{K}^\intercal\myvec{K} \myvec{X}$, which, however, would need to be solved iteratively.\newline
Instead, so called surrogate methods for computing NMF decompositions have been proposed recently and are introduced in the next section. They also consider alternating minimization steps for $\myvec{K}$ and $\myvec{X}$, but they replace the restricted minimization problems in \eqref{eq:AlternatingMinimzation1} and \eqref{eq:AlternatingMinimzation2} by simpler minimization tasks, which are obtained by locally replacing $F$ by surrogate functionals for $\myvec{K}$ and $\myvec{X}$ separately.
\section{Surrogate Functionals} \label{sec:Surrogate_functionals}
In this section, we discuss general surrogate approaches for minimizing general non-convex functionals, which are then exemplified for specific NMF functionals in later sections.\newline
Let us consider a general functional $F: \Omega \rightarrow \mathbb{R}$ where $\Omega \subset \mathbb{R}^N$ and the minimization problem
$$\min_{\myvec{x} \in \Omega} F (\myvec{x}). $$
We will later add suitable conditions guaranteeing the existence of minimizers or at least the existence of stationary points. Surrogate concepts replace this task by solving a sequence of comparatively simpler and convex surrogate functionals, which can be minimized efficiently. These methods are also commonly referred to as \textit{surrogate minimization} (or \textit{maximization}) \textit{algorithms} (SM) or also as \textit{MM algorithms}, where the first M stands for majorize and the second M for minimize (see also \cite{hunter,lange,zhang}). Such approaches have been demonstrated to be very useful in many fields of inverse problems, in particular for hyperspectral imaging \cite{demol}, medical imaging applications such as transmisson tomography \cite{defrise,fessler00} as well as MALDI imaging and tumor typing applications \cite{leuschner18}.\newline
Replacing a non-convex functional by a series of convex problems is the main motivation for such surrogate approaches. However, if constructed appropriately, they can also be used to replace non-differentiable functionals by a series of differentiable problems and they can be tailored such that gradient-descent methods for minimization yield multiplicative update rules which automatically incorporate non-negativity constraints without further projections.\newline
From this point on it is important to note that possible zero denominators during the derivation of the NMF algorithms as well as in the multiplicative update rules themselves will not be discussed explicitly throughout this work. Usually, this issue is handled in practice by adding a small positive constant in the denominator during the iteration scheme. In fact, the instability of NMF algorithms due to the convergence of some entries in the matrices to zero is not sufficiently discussed in the literature and still needs proper solution techniques. We will not focus on this problem and turn now to the basic definition and properties of surrogate functionals.
\subsection{Definitions and Basic Properties} \label{subsec:Definitions and basic properties}
As in \cite{leeseung}, we use the following definition of a surrogate functional.
\begin{definition}[Surrogate Functional] \label{def:Surrogat-Funktionen}
	Let $ \Omega\subseteq \mathbb{R}^N $ denote an open set and $ F: \Omega\rightarrow \mathbb{R}$ a functional defined on $\Omega$. Then $ Q_F : \Omega \times \Omega \rightarrow \mathbb{R} $ is called a \textbf{surrogate functional} or a \textbf{surrogate} for $ F, $ if it satisfies the following conditions:
	\begin{itemize}
		\item[(i)] $ Q_F(\myvec{x}, \myvec{a}) \geq F(\myvec{x})$ for all $ \myvec{x}, \myvec{a} \in \Omega $ \label{itm:def:Surrogat-Funktionen 1}
		\item[(ii)] $ Q_F(\myvec{x}, \myvec{x}) = F(\myvec{x}) $ for all $ \myvec{x}\in \Omega $ \label{itm:def:Surrogat-Funktionen 2}
	\end{itemize}
\end{definition}
This is the most basic definition, which does not require any  convexity or differentiability of the functional. However, it already allows to prove that the iteration
\begin{equation}
\myvec{x}^{[d+1]} \coloneqq \arg \min_{\myvec{x}\in \Omega} Q_F(\myvec{x}, \myvec{x}^{[d]})
\end{equation}
yields a sequence which  monotonically decreases $F$.
\begin{lemma}[Monotonic Decrease by Surrogate Functionals] \label{lem:Monotoner Abfall durch Surrogat-Funktion}
	Let $ \Omega\subseteq \mathbb{R}^N $ denote an open set, $ F: \Omega\rightarrow \mathbb{R}$ a given function  and $ Q_F $ a surrogate functional for $F$. Assume that $ \arg \min_{\myvec{x}\in \Omega} Q_F(\myvec{x}, \myvec{a}) $ is well defined for all $\myvec{a}\in \Omega $. Define the iterated updates by
	\begin{equation}\label{eq:lem:Monotoner Abfall durch Surrogat-Funktion:Regel}
	\myvec{x}^{[d+1]} \coloneqq \arg \min_{\myvec{x}\in \Omega} Q_F(\myvec{x}, \myvec{x}^{[d]})
	\end{equation}
	with $ \myvec{x}^{[0]} = \arg \min_{\myvec{x}\in \Omega} Q_F(\myvec{x}, \myvec{a}) $ for an arbitrary $ \myvec{a}\in \Omega. $ Then, $ F(\myvec{x}^{[d]})$ is a monotonoically decreasing sequence, i.e.
	\begin{equation}\label{eq:lem:Monotoner Abfall durch Surrogat-Funktion:Monotoner Abfall}
	F(\myvec{x}^{[d+1]})	\leq F(\myvec{x}^{[d]}).
	\end{equation}
\end{lemma}
\begin{proof}
	The monotone decrease   \eqref{eq:lem:Monotoner Abfall durch Surrogat-Funktion:Monotoner Abfall} follows directly from the defining properties of surrogate functionals, see Definition \ref{itm:def:Surrogat-Funktionen 1}: We obtain
	\begin{equation*}
	F(\myvec{x}^{[d+1]}) \leq Q_F(\myvec{x}^{[d+1]}, \myvec{x}^{[d]}) \overset{(\star)}{\leq} Q_F(\myvec{x}^{[d]}, \myvec{x}^{[d]}) = F(\myvec{x}^{[d]}),
	\end{equation*}
	where $ (\star) $ follows from the definition of $ \myvec{x}^{[d+1]} $ in \eqref{eq:lem:Monotoner Abfall durch Surrogat-Funktion:Regel}. \qed
\end{proof}
\begin{figure}
	%		\captionsetup{width=\textwidth, format=plain}
	\centering
	\includegraphics[width=\textwidth]{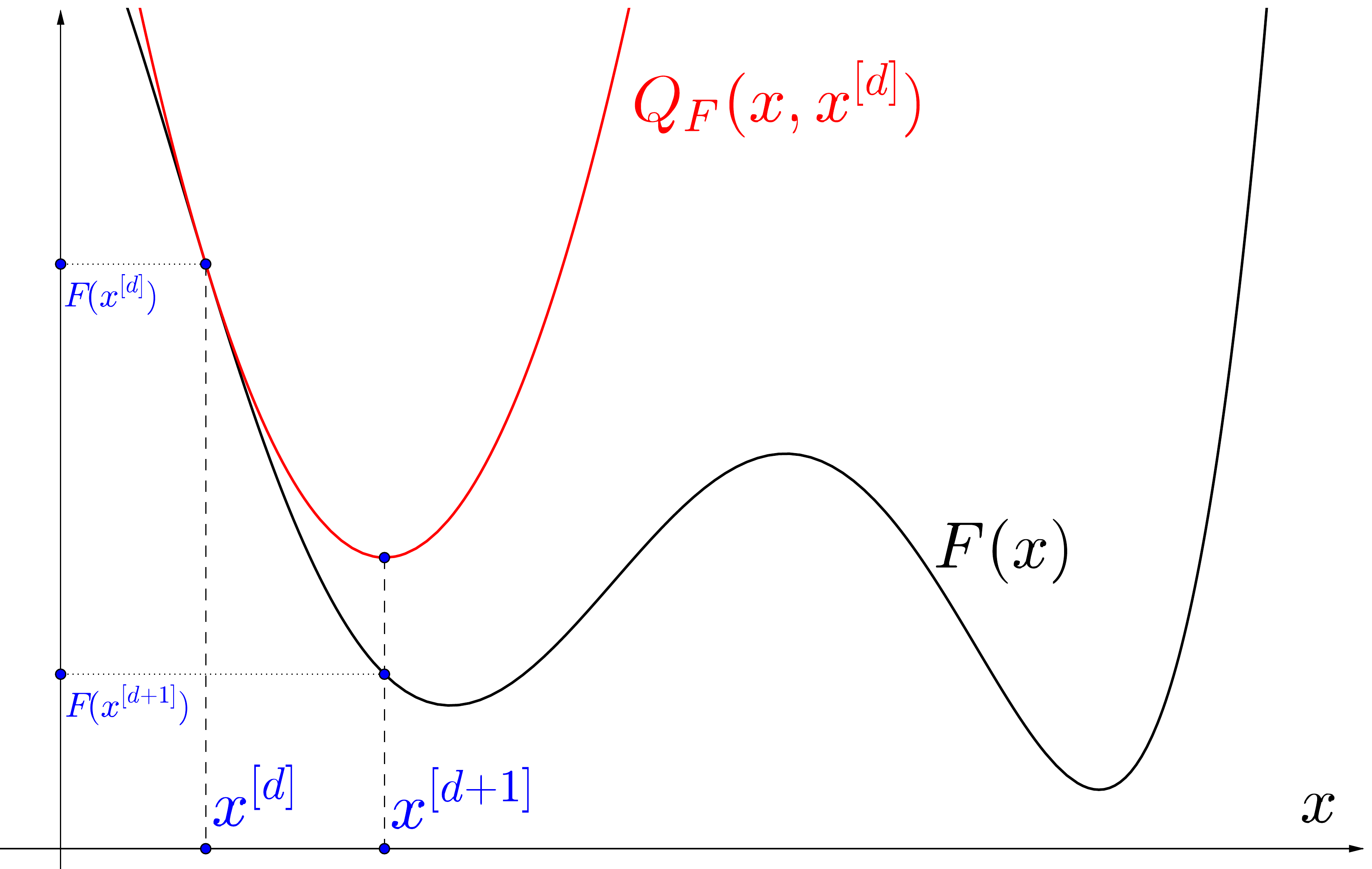}
	\caption{Visualization of the surrogate principle for non-convex $ F $ with convex 
		surrogate functional $ Q_F $ according to Lemma \ref{lem:Monotoner Abfall durch Surrogat-Funktion}}
	\label{fig:Prinzip der Minimierung einer Kostenfunktion durch eine Surrogat-Funktion}
\end{figure}

\begin{remark}[Addition of Surrogate Functionals] \label{rem:Addition von Surrogat-Funktionen}
	Let $ \Omega\subseteq \mathbb{R}^n $ be an open set, $ F, G:\Omega\rightarrow\mathbb{R} $ pointwise defined functionals  and $ Q_F, Q_G $ corresponding surrogates. Then $ Q_F + Q_G $ is a surrogate functional for $ F+G. $
\end{remark}
For each functional $F$ there typically exist a large variety of surrogate functionals and we can aim at optimizing the structure of surrogate functionals. The following additional property is the key to simple and efficient minimization schemes for surrogate functionals.

\begin{definition}[Separability of a Surrogate Functional] \label{def:Separabilitaet}
	Let $ \Omega\subseteq \mathbb{R}^N $ denote an open set, $ F: \Omega\rightarrow \mathbb{R}$ a functional and $ Q_F $ a surrogate for $F$. The surrogate $ Q_F$ is called \textbf{separable}, if there exist functions $g_i:\mathbb{R} \times \Omega \rightarrow \mathbb{R},$ such that  %$ g_i: \Omega_i \times \Omega \rightarrow \mathbb{R}$ für $ i\in \{1, \dots, n \} $ existieren, mit %Ferner sei die Abbildung  $\pi_i: \Omega \rightarrow\mathbb{R}$ mit $ \pi_i(\myvec{x}) \colo\mathbb{R}neqq x_i$ die Projektion auf die $ i $-te Komponente mit Bild $ \Omega_i\coloneqq \pi_i(\Omega). $ 
	\begin{equation}\label{eq:def:Separabilität}
	Q_F(\myvec{x}, \myvec{a}) = \sum_{i=1}^{N} g_i(x_i, \myvec{a}) \quad \forall \myvec{x}, \myvec{a}\in \Omega.
	\end{equation}
\end{definition}
Lemma \ref{lem:Monotoner Abfall durch Surrogat-Funktion} above only ensures the monotonic decrease of the cost functional, which is not sufficient to guarantee convergence of the sequence $\{ \myvec{x}^{[d]}\}$ to a minimizer of $F$ or at least to a stationary point of $F$. The convergence theory for surrogate functionals is far from being complete (see also the works \cite{lange} and \cite{zhang}).\newline
%For a given cost functional $F$ and a corresponding surrogate $Q_F$,  proves convergence of the iterates to a stationary point under the following assumptions on $ F$:
%\begin{itemize}
%	\item $ F $ is two times differentiable and bounded from below
%	\item The Hessian of $ F $ is positive definite
%	\item The stationary points of $ F $ are isolated
%\end{itemize}
%%for twice differentiable functionals $F$ with positive definite Hessian,   In addition this paper assumes, that $F$ is bounded from below and that the stationary points of $F$ are isolated.
%%These results uses some classical results on convex minimization, see \cite{wu}, and utilizes compactness of the level sets of $F$ in connection with the Liapunov theory for discrete dynamical systems.
%A convergence result, which only requires first order differentiability for $F$ but additional smoothness assumptions on $Q_F$ can be found in .\newline
Despite this lack of theoretical foundation, surrogate based minimization yields strictly decreasing sequences for a large variety of applications. In particular, surrogate based methods can be constructed such that first order optimality conditions lead to multiplicative update rules, which  - in view of applications to NMF constructions - is a very desirable property.\newline
We now turn to discussing three different construction principles for surrogate functionals.

\subsection{Jensen's Inequality} \label{subsec:Jensen's inequality}
The starting point is the well known Jensen's inequality for convex functions, see \cite{Cvetkovski12}.
\begin{lemma}[Jensen's Inequality]
	Let $ \Omega \subseteq \mathbb{R}^N $ denote a convex set, $ F:\Omega \rightarrow \mathbb{R} $ a convex function and $ \lambda_i \in [0, 1] $ non-negative numbers for $ i\in \{1, \dots, k\} $ with $ \sum_{i=1}^{k} \lambda_i = 1. $ Then for all $ \myvec{x}_i\in \Omega, $ it holds that
	\begin{equation}\label{eq:lem:Jensensche Ungleichung}
	F\left (\sum_{i=1}^{k} \lambda_i \myvec{x}_i\right ) \leq \sum_{i=1}^{k} \lambda_i F(\myvec{x}_i).
	\end{equation}
\end{lemma}
In this subsection we consider functionals $F$ which are derived from continuously differentiable and convex functions 
$ f:\mathbb{R}_{>0} \rightarrow \mathbb{R} $
via
\begin{equation} \label{eq:Jensensche Ungleichung:Ansatz}
\begin{aligned}
F: \Omega &\rightarrow \mathbb{R} \\
\myvec{v} &\mapsto f(\myvec{c}^\intercal \myvec{v}) \ \ .
\end{aligned}
\end{equation}
for $ \Omega \subseteq \mathbb{R}_{\geq 0}^N $ and some auxiliary variable $ \myvec{c}\in \Omega. $ This also implies, that $F$ is convex, since

\begin{align*}
F(\myvec{v})&\geq F(\tilde{\myvec{v}}) + \nabla F(\tilde{\myvec{v}})^\intercal (\myvec{v}-\tilde{\myvec{v}}) \\
\Leftrightarrow f(\myvec{c}^\intercal \myvec{v}) &\geq f(\myvec{c}^\intercal \tilde{\myvec{v}}) + f'(\myvec{c}^\intercal \tilde{\myvec{v}}) ( \myvec{c}^\intercal \myvec{v} - \myvec{c}^\intercal \tilde{\myvec{v}}).
\end{align*}
We now choose $ \lambda_i\in [0, 1] $ with $ \sum_{i=1}^{N} \lambda_i =1$ and $ \myvec{\alpha} \in \mathbb{R}^N $  and define
\begin{align}
\lambda_i &\coloneqq \dfrac{c_i b_i}{\myvec{c}^\intercal \myvec{b}} \label{eq:Jensensche Ungleichung:lambda_i}\\
\alpha_i &\coloneqq \dfrac{c_iv_i}{\lambda_i} \left ( = \dfrac{v_i \myvec{c}^\intercal \myvec{b}}{b_i} \right ) \label{eq:Jensensche Ungleichung:alpha_i}
\end{align}
for some $ \myvec{b}\in \Omega. $ This implies
\begin{equation} \label{eq:Jensensche Ungleichung:Konstruktion}
F(\myvec{v}) = f(\myvec{c}^\intercal \myvec{v})  = f\left (\sum_{i=1}^{N} \lambda_i \alpha_i \right ) \leq \sum_{i=1}^{N} \dfrac{c_i b_i}{\myvec{c}^\intercal \myvec{b}} f\left (\dfrac{\myvec{c}^\intercal \myvec{b}}{b_i} v_i \right ) \eqqcolon Q_F(\myvec{v}, \myvec{b}).
\end{equation}
The functional $ Q_F: \Omega \times \Omega \rightarrow \mathbb{R} $ defines a surrogate for $F$, which can be seen by the inequality above and by observing
\begin{equation*}
Q_F(\myvec{v}, \myvec{v}) = \sum_{i=1}^{N} \dfrac{c_i v_i}{\myvec{c}^\intercal \myvec{v}} f(\myvec{c}^\intercal \myvec{v}) = f(\myvec{c}^\intercal \myvec{v}) = F(\myvec{v}).
\end{equation*}

\subsection{Low Quadratic Bound Principle}\label{subsec:Low quadratic bound principle}
This concept is based on a Taylor expansion of $F$ in combination with a majorization of the quadratic term. This so called low quadratic bound principle (LQBP) has been introduced in \cite{lindsay} and was used in particularly for the computation of maximum-likelihood estimators. These methods do not require that $F$ itself is convex and its construction is based on the following lemma.

\begin{lemma}[Low Quadratic Bound Principle] \label{lem:Low Quadratic Bound Principle}
	Let $ \Omega \subseteq \mathbb{R}^N $ denote an open and convex set  and $ f:\Omega \rightarrow \mathbb{R} $ a twice  differentiable functional. Assume that a matrix  $ \myvec{\Lambda}(\myvec{x}) \in \mathbb{R}^{N\times N} $ exists, such that  $ \myvec{\Lambda}(\myvec{x}) - \nabla^2f(\myvec{x}) $ is positive semi-definite for all $ \myvec{x}\in \Omega. $ We then obtain a \textit{quadratic majorization}
	\begin{align}
	f(\myvec{x}) &\leq f(\myvec{a}) + \nabla f(\myvec{a})^\intercal (\myvec{x}-\myvec{a}) + \dfrac{1}{2} (\myvec{x}-\myvec{a})^\intercal \myvec{\Lambda}(\myvec{a}) (\myvec{x}-\myvec{a}) \quad \forall \myvec{x}, \myvec{a} \in \Omega \label{eq:lem:Qudratische Majorisierung} \\
	&\eqqcolon Q_f(\myvec{x}, \myvec{a}), \nonumber
	\end{align}
	and $ Q_f $ is a surrogate functional for  $ f. $
\end{lemma}
\begin{proof}
	The proof of this classical result is based on the second-order Taylor polynomial of $ f $ and shall be left to the reader. \qed
\end{proof}
The related update rule for surrogate minimization can be stated explicitly under natural assumptions on the matrix $\myvec{\Lambda}.$
\begin{corollary}\label{cor:Update zum Low Quadratic Bound Principle}
	Assume that the assumptions of Lemma \ref{lem:Low Quadratic Bound Principle} hold. In addition, assume that  $ \myvec{\Lambda} $ is a positive definite and symmetric matrix. Then, the corresponding surrogate $ Q_f $ is strictly convex in its first variable and we have from \eqref{eq:lem:Monotoner Abfall durch Surrogat-Funktion:Regel}
	\begin{align}
	\myvec{x}^{[d+1]} &= \arg \min_{\myvec{x}\in \Omega} Q_f(\myvec{x}, \myvec{x}^{[d]}) \nonumber \\
	&= 	\myvec{x}^{[d]} - \myvec{\Lambda}^{-1}(\myvec{x}^{[d]}) \nabla f(\myvec{x}^{[d]}). \label{eq:cor:Update zum Low Quadratic Bound Principle:Update}
	\end{align} \vspace*{-4ex}
\end{corollary}
\begin{proof}
	For an arbitrary $ \alpha\in \{1, \dots, N \}, $ we have that
	\begin{align*}
	&\dfrac{\partial Q_f}{\partial x_\alpha}(\myvec{x}, \myvec{a}) = \dfrac{\partial }{\partial x_\alpha} \left ( \sum_{i=1}^{N} \dfrac{\partial f}{\partial x_i}(\myvec{a}) (x_i - a_i) + \dfrac{1}{2} \sum_{i=1}^{N} \sum_{j=1}^{N} \Lambda(\myvec{a})_{ij} \cdot (x_i - a_i) (x_j - a_j) \right ) \nonumber \\
	&= \dfrac{\partial f}{x_\alpha}(\myvec{a}) + \dfrac{1}{2} \sum_{\substack{i=1 \\ i\neq \alpha}}^{N} \Lambda(\myvec{a})_{\alpha i} (x_i - a_i) + \dfrac{1}{2} \sum_{\substack{i=1 \\ i\neq \alpha}}^{N} \Lambda(\myvec{a})_{i \alpha} (x_i - a_i) + \Lambda(\myvec{a})_{\alpha \alpha} (x_\alpha - a_\alpha) \nonumber \\
	&\overset{(\star)}{=} \dfrac{\partial f}{x_\alpha}(\myvec{a}) + \myvec{\Lambda}(\myvec{a})_{\alpha, \bullet} (\myvec{x}-\myvec{a}),
	\end{align*}
	where $ (\star) $ utilizes 	the symmetry of $ \myvec{\Lambda}. $ Hence, it holds that $ \nabla_{\myvec{x}} Q_f(\myvec{x}, \myvec{a}) = \nabla f(\myvec{a}) + \myvec{\Lambda}(\myvec{a}) (\myvec{x}-\myvec{a}).$ The Hessian matrix of $ Q_f $ then satisfies
	\begin{equation*}
	\nabla_{\myvec{x}}^2 Q_f(\myvec{x}, \myvec{a}) = \myvec{\Lambda}(\myvec{a}).
	\end{equation*}
	This implies the positive definiteness of the functional, hence, it has a uniqie minimizer, which is given by
	\begin{align*}
	\nabla_{\myvec{x}} Q_f(\myvec{x}^*_{\myvec{a}}, \myvec{a}) = 0 &= \nabla f(\myvec{a}) + \myvec{\Lambda}(\myvec{a}) (\myvec{x}^*_{\myvec{a}}-\myvec{a}) \\
	\Leftrightarrow \myvec{x}^*_{\myvec{a}} &= \myvec{a} - \myvec{\Lambda}^{-1}(\myvec{a}) \nabla f(\myvec{a}).
	\end{align*}
	This is the update rule above. \qed
\end{proof}
The computation of the inverse of $\myvec{\Lambda}$ is particularly simple if $\myvec{\Lambda}$ is a diagonal matrix. Furthermore, the diagonal structure ensures the separability of the surrogate functional mentioned in Definition \ref{def:Separabilitaet}. Therefore, we consider matrices of the form
\begin{equation}\label{eq:Low Quadratic Bound Principle:Lambda}
\Lambda (\myvec{a})_{i i} \coloneqq \dfrac{(\nabla^2 f (\myvec{a})\ a)_i + \kappa_i}{a_i},
\end{equation}
where $\kappa_i \ge 0$ has to be chosen individually depending on the considered cost function. We will see that an appropriate choice of $ \kappa_i $ will lead finally to the desired multiplicative update rules of the NMF algorithm. \newline %We will see in the case of the Frobenius discrepancy term, that the appropriate choice of $ \kappa_i $ depends on the $ \ell_1 $ regularization terms of the cost function, which will lead finally to the desired multiplicative update rules.
The matrix $ \myvec{\Lambda}(\myvec{a}) $ in \eqref{eq:Low Quadratic Bound Principle:Lambda} fulfills the conditions in Corollary \ref{cor:Update zum Low Quadratic Bound Principle} as it can be seen by the following lemma. Therefore, if $ \myvec{\Lambda} $ is constructed as in \eqref{eq:Low Quadratic Bound Principle:Lambda}, the update rule in \eqref{eq:cor:Update zum Low Quadratic Bound Principle:Update} can be applied immediately.
\begin{lemma}
	Let $ \myvec{M}\in \mathbb{R}_{\geq 0}^{N \times N} $ denote a symmetric matrix. With $ \myvec{a}\in \mathbb{R}_{>0}^N $ and $ \kappa_i\geq 0, $  we define the diagonal matrix $ \myvec{\Lambda}, $ such that
	\begin{equation}\label{eq:rem:Low Quadratic Bound Principle:Lambda}
	\Lambda_{i i} \coloneqq \dfrac{(Ma)_i + \kappa_i}{a_i}
	\end{equation}
	for $ i=1,\dots,N. $ Then $ \myvec{\Lambda} $ and $ \myvec{\Lambda}-\myvec{M} $ are positive semi-definite.
\end{lemma}

\begin{proof}
	Let $ \myvec{\zeta}\in \mathbb{R}^N $ denote an arbitrary vector and let $ \myvec{\delta} $ denote the Kronecker symbol. Then
	\begin{align*}
	\myvec{\zeta}^\intercal (\myvec{\Lambda}-\myvec{M}) \myvec{\zeta} &=\sum_{i, j=1}^{N} \zeta_i \delta_{ij} \dfrac{(Ma)_i + \kappa_i}{a_i} \zeta_j - \sum_{i, j=1}^{N} \zeta_i M_{ij} \zeta_j \\
	&=\sum_{i=1}^{N} \zeta_i^2 \dfrac{(Ma)_i}{a_i} + \zeta_i^2 \dfrac{\kappa_i}{a_i} - \sum_{i, j=1}^{N} \zeta_i M_{ij} \zeta_j \\
	&\geq \sum_{i, j=1}^{N} \zeta_i^2 \dfrac{a_j}{a_i} M_{ij} - \sum_{i, j=1}^{N} \zeta_i M_{ij} \zeta_j \\
	&=\sum_{i=1}^{N} \zeta_i^2 M_{ij} + \sum_{\substack{i, j=1 \\ i<j}}^{N} \left( \zeta_i^2 \dfrac{a_j}{a_i} + \zeta_j^2 \dfrac{a_i}{a_j}\right) M_{ij} - \sum_{i, j=1}^{N} \zeta_i M_{ij} \zeta_j\displaybreak \\
	&=\sum_{i, j=1}^{N} \left[ \dfrac{1}{2} \zeta_i^2 \dfrac{a_j}{a_i} M_{ij} + \dfrac{1}{2} \zeta_j^2 \dfrac{a_i}{a_j} M_{ij} - \zeta_i M_{ij} \zeta_j \right] \\
	&=\dfrac{1}{2} \sum_{i, j=1}^{N} \zeta_i^2 \dfrac{a_j}{a_i} M_{ij} + \zeta_j^2 \dfrac{a_i}{a_j} M_{ij} - 2 \sqrt{\dfrac{a_j}{a_i}} \sqrt{\dfrac{a_i}{a_j}} \zeta_i M_{ij} \zeta_j \\
	&=\dfrac{1}{2} \sum_{i, j=1}^{N} \left ( \sqrt{\dfrac{a_j}{a_i}} \zeta_i - \sqrt{\dfrac{a_i}{a_j}} \zeta_j \right )^2 M_{ij} \geq 0.
	\end{align*}
	The positive semi-definiteness of $ \myvec{\Lambda} $ follows from its diagonal structure. \qed
\end{proof}

\subsection{Further Construction Principles}\label{subsec:Further construction principles}
So far we have discussed two major construction principles based on either Jensen's inequality or on upper bounds for the quadratic term in Taylor expansions.
\cite{lange} lists further construction principles, which however will not be used for NMF constructions in the subsequent sections of this paper. For completeness we shortly list their main properties.\newline
A relaxation of the approach based on Jensen's inequality is achieved by choosing
$ \alpha_i \geq 0, i\in \{1, \dots, N \} $ such that $ \sum_{i=1}^{N} \alpha_i = 1 $ and $ \alpha_i>0 $ if $ c_i\neq 0 $, which yields
\begin{equation*}
F(\myvec{v}) = f(\myvec{c}^\intercal \myvec{v})  \leq \sum_{i=1}^{N} \alpha_i f \left ( \dfrac{c_i}{\alpha_i} (v_i - b_i) + \myvec{c}^\intercal \myvec{b} \right ) \eqqcolon Q_F(\myvec{v}, \myvec{b}).
\end{equation*}
A typical choice is
\begin{equation*}
\alpha_i \coloneqq \dfrac{\vert c_i \vert^p}{\sum_{j=1}^{n} \vert c_j \vert^p}
\end{equation*}
which leads to surrogate functionals for $p \geq 0$. This type of surrogate was originally introduced in the context of medical imaging, see \cite{depierro}, for positron emission tomography.
\newline
Another approach is based on combining arithmetic with geometic means and can be used for constructing surrogates for posynomial functions. For $ \myvec{\alpha}, \myvec{v}, \myvec{a}\in \mathbb{R}_{>0}^N,$ we obtain
\begin{align*}
F(\myvec{v}) = \prod_{i=1}^{N} v_i^{\alpha_i} \leq \left (\prod_{i=1}^{N}a_i^{\alpha_i} \right ) \sum_{i=1}^{N} \dfrac{\alpha_i}{\sum_{k=1}^{N} \alpha_i} \left ( \dfrac{v_i}{a_i} \right )^{\sum_{k=1}^{N} \alpha_i} \eqqcolon Q_F(\myvec{v}, \myvec{a})
\end{align*}
\section{Surrogates for NMF Functionals} \label{sec:surrogates_for_NMF_functionals}
In this section we apply the general construction principles of Section \ref{sec:Surrogate_functionals} to the NMF problem as stated in \eqref{eq:def:NMF-Minimumproblem:Minimumproblem}. The resulting functional $F(\myvec{K},\myvec{X})$ depends on both factors of the matrix decomposition and minimization is attempted by alternating minimization with respect to $\myvec{K}$ and $\myvec{X}$ as described in \eqref{eq:AlternatingMinimzation1} and \eqref{eq:AlternatingMinimzation2}. \newline
However, we replace the functional $F$ in each iteration by suitable surrogate functionals, which allow an explicit minimization. Hence, we avoid the minimization of $F$ itself, which even for the most simple quadaratic formulation requires to solve a high dimensional linear system.\newline
We start by considering the discrepancy terms for $\beta =2$ (Frobenius norm) and $\beta =1$ (Kullback-Leibler divergence) and determine surrogate functionals with respect to $\myvec{X}$ and $\myvec{K}$. We then add several penalty terms and develop surrogate functionals accordingly.
With regard to the construction of surrogates for the case of $ \beta=0, $ which leads to the so-called \textit{Itakura-Saito divergence}, we refer to the works \cite{FBD09,fevotte,sun}.
\subsection{Frobenius Discrepancy and Low Quadratic Bound Principle}\label{subsec:Frobenius discrepancy and LQBP}
We start by constructing a surrogate for the minimization with respect to $\myvec{X}$ for the Frobenius discrepancy
\begin{equation} \label{eq:Surrogat-Funktionen mit LQBP (Frobenius):Kostenfunktional X}
F(\myvec{X})\coloneqq \dfrac{1}{2} \Vert \myvec{Y}-\myvec{KX}\Vert_F^2.
\end{equation}
Let $\myvec{Y}_{\bullet, j}$, resp. $\myvec{X}_{\bullet, j}$, denote the column vectors of $\myvec{Y}$, resp. $\myvec{X}$. 
The separability of $F$ yields
\begin{equation}
F(\myvec{X}) = \dfrac{1}{2}\sum_{j=1}^{m} \Vert \myvec{Y}_{\bullet, j}-\myvec{K}\myvec{X}_{\bullet, j} \Vert^2 \eqqcolon \sum_{j=1}^{m} f_{\myvec{Y}_{\bullet, j}}(\myvec{X}_{\bullet, j}), \label{eq:Frobenius-Norm:Surrogat-Funktion mit LQBP:F}
\end{equation}
Hence, the minimization separates for the different $f_{\myvec{Y_{\bullet, j}}}$ terms. The Hessian of these terms is given by $$\nabla^2 f_{\myvec{Y_{\bullet, j}}}(\myvec{a})=\myvec{K}^\intercal\myvec{K}$$ and the LQBP construction principle of the previous section with $\kappa_k=0$ yields
\begin{equation*}
\Lambda_{f_{\myvec{Y_{\bullet, j}}}}(\myvec{a})_{kk} = \dfrac{(K^\intercal K a)_k}{a_k},
\end{equation*}
leading to the surrogate functionals
\begin{equation}\label{eq:Surrogat-Funktionen mit LQBP (Frobenius):Surrogat x}
Q_{f_{\myvec{Y_{\bullet, j}}}}(\myvec{x}, \myvec{a}) = f_{\myvec{Y_{\bullet, j}}}(\myvec{a}) + \nabla f_{\myvec{Y_{\bullet, j}}}(\myvec{a})^\intercal (\myvec{x}-\myvec{a}) + \dfrac{1}{2} (\myvec{x}-\myvec{a})^\intercal \myvec{\Lambda}_{f_{\myvec{Y_{\bullet, j}}}}(\myvec{a}) (\myvec{x}-\myvec{a}).
\end{equation}
An appropriate choice of $ \kappa_k $ ensures the multiplicativity of the final NMF algorithm. In the case of the Frobenius discrepancy term, we will see that suitable $ \kappa_k $ can be chosen dependent on $ \ell_1 $ regularization terms in the cost function, which are not included up to now (see Subsection \ref{subsec:Surrogates for lp-penalty terms} and Appendix \ref{app:subsec:Frobenius-norm} for more details on this issue). Due to the absent $ \ell_1 $ terms, we set $ \kappa_k=0 $ to get the desired multiplicative update rules.
Summing up the contributions of the columns of $\myvec{X}$ yields the final surrogate
\begin{equation*}
Q_F:\mathbb{R}^{p\times m} \times \mathbb{R}^{p\times m} \rightarrow \mathbb{R}, \quad (\myvec{X}, \myvec{A}) \mapsto \sum_{j=1}^{m} Q_{f_{\myvec{Y_{\bullet, j}}}}(\myvec{X}_{\bullet, j}, \myvec{A}_{\bullet, j})
\end{equation*}
The equivalent construction for $\myvec{K}$ can be obtained by regarding the rows of $\myvec{K}$ separately, which for 
\begin{equation*}
g_{\myvec{y}}: \mathbb{R}_{\geq 0}^p \rightarrow \mathbb{R}, \quad \myvec{k}\mapsto \dfrac{1}{2} \Vert \myvec{y}-\myvec{kX}\Vert^2
\end{equation*}
yields $\nabla^2 g_{\myvec{y}}(\myvec{a})=\myvec{XX}^\intercal$. Putting 
\begin{equation*}
\Lambda_{g_{\myvec{y}}}(\myvec{a})_{kk} = \dfrac{(aXX^\intercal)_k}{a_k}
\end{equation*}
leads to the surrogate
\begin{equation*}
Q_{g_{\myvec{y}}}(\myvec{k}, \myvec{a}) = g_{\myvec{y}}(\myvec{a}) + (\myvec{k}-\myvec{a}) \nabla g_{\myvec{y}}(\myvec{a}) + \dfrac{1}{2} (\myvec{k}-\myvec{a}) \myvec{\Lambda}_{g_{\myvec{y}}}(\myvec{a}) (\myvec{k}-\myvec{a})^\intercal.
\end{equation*}
We summarize this surrogate construction in the following theorem.
\begin{theorem}[Surrogate Functional for the Frobenius Norm with LQBP] \label{satz:Surrogat-Funktionen für Frobenius-Norm mit dem LQBP}
	We consider the cost functionals $ F(\myvec{X}) \coloneqq \nicefrac{1}{2} \Vert \myvec{Y}-\myvec{KX}\Vert_F^2 $ and $ G(\myvec{K}) \coloneqq \nicefrac{1}{2} \Vert \myvec{Y}-\myvec{KX}\Vert_F^2. $ Then 
	\begin{align}
	Q_{F, 1}(\myvec{X}, \myvec{A}) &= \sum_{j=1}^{m} Q_{f_{\myvec{Y}_{\bullet, j}}}(\myvec{X}_{\bullet, j}, \myvec{A}_{\bullet, j}),  \label{eq:satz:Surrogat-Funktionen für Frobenius-Norm mit dem LQBP:X}\\
	Q_{G, 1}(\myvec{K}, \myvec{A}) &= \sum_{i=1}^{n} Q_{g_{\myvec{Y}_{i, \bullet}}}(\myvec{K}_{i, \bullet}, \myvec{A}_{i, \bullet}) \label{eq:satz:Surrogat-Funktionen für Frobenius-Norm mit dem LQBP:K}
	\end{align}
	define separable and convex surrogate functionals.
\end{theorem}

\subsection{Frobenius Discrepancy and Jensen's Inequality} \label{subsec: Frobenius_discrepancy_and_Jensens_inequality}
Again we focus on deriving a surrogate functional for $\myvec{X}$, the construction for $\myvec{K}$ will be very similar. Expanding the Frobenius discrepancy yields
\begin{equation*}
F(\myvec{X}) \coloneqq \dfrac{1}{2} \Vert \myvec{Y}-\myvec{KX}\Vert_F^2 = \dfrac{1}{2} \sum_{i=1}^{n} \sum_{j=1}^{m} (Y_{ij} - (KX)_{ij})^2.
\end{equation*}
%Let $K_{(j)}$, resp. $X^{(k)}$, denote the $j$-the row of $K$, resp. the $k$-the column of $X$.
Putting $ \myvec{v}\coloneqq \myvec{X}_{\bullet, j} \in \mathbb{R}_{\geq 0}^p$ and $ \myvec{c}\coloneqq {\myvec{K}_{i, \bullet}}^\intercal \in \mathbb{R}_{\geq 0}^p $ allows us to define

\begin{equation} \label{eq:Surrogat-Funktionen mit Jensenscher Ungleichung:X:f}
f:\mathbb{R}_{\geq 0} \rightarrow \mathbb{R} \quad \text{with} \quad f(t)\coloneqq (Y_{ij} - t)^2,
\end{equation}
such that
\begin{equation}\label{eq:Surrogat-Funktionen mit Jensenscher Ungleichung (Frobenius):Gleichung zu f}
f(\myvec{c^\intercal \myvec{v}}) = (Y_{ij} - (KX)_{ij})^2.
\end{equation}
Hence we have separated the Forbenius discrepancy suitably for applying Jensen's inequality. Following the construction principle in Subsection \ref{subsec:Jensen's inequality}, we define
\begin{align}
\lambda_k &= \dfrac{K_{ik}A_{kj}}{(KA)_{ij}}, \label{eq:Surrogat-Funktionen mit Jensenscher Ungleichung:X:lambda_i} \\
\alpha_k   &= \dfrac{K_{ik}X_{kj}}{\lambda_k}. \label{eq:Surrogat-Funktionen mit Jensenscher Ungleichung:X:alpha_i}
\end{align}
with the auxiliary variable $ \myvec{A}\in \mathbb{R}_{\geq 0}^{p\times m} $ and  $ \myvec{b}\coloneqq \myvec{A}_{\bullet, j}\in \mathbb{R}_{\geq 0}^p, $ which yields the inequality
\begin{equation*}
(Y_{ij} - (KX)_{ij})^2 \leq \sum_{k=1}^{p} \dfrac{K_{ik}A_{kj}}{(KA)_{ij}} \left(Y_{ij} - \dfrac{X_{kj}}{A_{kj}}(KA)_{ij}\right)^2.
\end{equation*}
Inserting this into the decomposition of the Frobenius discrepancy yields the surrogate $Q_{F, 2}(\myvec{X}, \myvec{A})$ by
\begin{equation*}
F(\myvec{X}) \leq \dfrac{1}{2} \sum_{i=1}^{n} \sum_{j=1}^{m} \dfrac{1}{(KA)_{ij}} \sum_{k=1}^{p} K_{ik}A_{kj} \left(Y_{ij} - \dfrac{X_{kj}}{A_{kj}}(KA)_{ij}\right)^2 \eqqcolon Q_{F, 2}(\myvec{X}, \myvec{A}),
\end{equation*}
The construction of a surrogate for $\myvec{K}$ proceeds in the same way. We summarize the results in the following theorem.

\begin{theorem}[Surrogate Functional for the Frobenius Norm with Jensen's Inequality] \label{satz:Surrogat-Funktionen für Frobenius-Norm mit der JU}
	We consider the cost functionals $ F(\myvec{X}) \coloneqq \nicefrac{1}{2} \Vert \myvec{Y}-\myvec{KX}\Vert_F^2 $ and $ G(\myvec{K}) \coloneqq \nicefrac{1}{2} \Vert \myvec{Y}-\myvec{KX}\Vert_F^2 $. Then 
	\begin{align}
	Q_{F, 2}(\myvec{X}, \myvec{A}) &\coloneqq \dfrac{1}{2} \sum_{i=1}^{n} \sum_{j=1}^{m} \dfrac{1}{(KA)_{ij}} \sum_{k=1}^{p} K_{ik}A_{kj} \left(Y_{ij} - \dfrac{X_{kj}}{A_{kj}}(KA)_{ij}\right)^2 \label{eq:satz:Surrogat-Funktionen für Frobenius-Norm mit der JU:X}\\
	Q_{G, 2}(\myvec{K}, \myvec{A}) &\coloneqq \dfrac{1}{2} \sum_{i=1}^{n} \sum_{j=1}^{m} \dfrac{1}{(AX)_{ij}} \sum_{k=1}^{p} A_{ik}X_{kj} \left( Y_{ij} - \dfrac{K_{ik}}{A_{ik}}(AX)_{ij}\right)^2 \label{eq:satz:Surrogat-Funktionen für Frobenius-Norm mit der JU:K}
	\end{align}
	define separable and convex surrogate functionals.
\end{theorem}
These surrogates are equal to the ones proposed in \cite{demol}. We will later use first order necessary conditions of the surrogate functionals for obtaining algorithms for minimization. We already note
\begin{align*}
\dfrac{\partial Q_{F, 1}}{\partial X_{\alpha \beta}} = \dfrac{\partial Q_{F, 2}}{\partial X_{\alpha \beta}} \quad \text{and} \quad \dfrac{\partial Q_{G, 1}}{\partial K_{\alpha \beta}} = \dfrac{\partial Q_{G, 2}}{\partial K_{\alpha \beta}},
\end{align*}
i.e. despite the rather different derivations, the update rules for the surrogates obtained by LQBP and Jensen's inequality will be identical. 

\subsection{Surrogates for Kullback-Leibler Divergence}\label{subsec:Surrogates for Kullback-Leibler divergence}
The case $\beta = 1$ in Definition \ref{def:beta-divergence} yields the so-called Kullback-Leibler divergence (KLD). For matrices $ \myvec{M}, \myvec{N}\in \mathbb{R}^{n\times m}_{>0}, $ it is defined as
\begin{equation}\label{def:Kullback-Leibler-Divergenz}
\KL (\myvec{M}, \myvec{N}) \coloneqq  D_1(\myvec{M}, \myvec{N}) = \sum_{i=1}^{n} \sum_{j=1}^{m} M_{i j} \log \left (\dfrac{M_{i j}}{N_{i j}} \right ) - M_{i j} + N_{i j}
\end{equation}
and has been investigated intensively in connection with non-negative matrix factorization methods \cite{demol,fevotte,lecharlier13,leeseung}. In our context, we define the cost functional for the NMF construction by
$$F(\myvec{X},\myvec{K})\coloneqq \KL(\myvec{Y},\myvec{KX}).$$
%We can apply the LPBQ-principle for constructing surrogates for $\myvec{K}$ and $\myvec{X}$. However, the resulting functionals do not yield to easier minimization techniques, i.e. the first order necessary conditions for minimizers of the resulting surrogate functionals do  not allow an explicit solution.\newline
We will focus in this subsection on Jensen's inequality for constructing surrogates for the KLD since they will lead to the known classical NMF algorithms (see also \cite{demol,lecharlier13,leeseung}). However, it is also possible to use the LQBP principle to construct a suitable surrogate functional for the KLD which leads to different, multiplicative update rules (see Appendix \ref{app:sec:KLD_and_LQBP}).\newline
We start by deriving a surrogate for the minimization with respect to $\myvec{X}$, i.e. we consider
\begin{equation*}
F(\myvec{X}) \coloneqq \KL(\myvec{Y}, \myvec{KX}) = \sum_{i=1}^{n} \sum_{j=1}^{m} Y_{ij} \log (Y_{ij}) -  Y_{ij} \log ((KX)_{ij}) - Y_{ij} + (KX)_{ij}.%\sum_{j=1}^{m} \sum_{i=1}^{n} (Y^{(k)})_j \log ((Y^{(k)})_j) -  (Y^{(k)})_j \log ((KX^{(k)})_j) - (Y^{(k)})_j + (KX^{(k)})_j.
\end{equation*}
Using the same $ \lambda_k $ and $ \alpha_k $ as in the section above and applying it to the convex function $f(t)\coloneqq -\ln(t)$, we obtain
\begin{align*}
-\ln ((KX)_{ij}) \leq -\sum_{k=1}^{p} \dfrac{K_{ik}A_{kj}}{(KA)_{ij}} \ln \left( \dfrac{X_{kj}}{A_{kj}}(KA)_{ij}\right).
\end{align*}
Multiplication with $Y_{ij} \ge 0$ and the addition of appropriate terms yields
\begin{align*}
F(\myvec{X}) &\leq\text{\begin{small} $ \sum_{i=1}^{n} \sum_{j=1}^{m} Y_{ij} \ln(Y_{ij}) - Y_{ij} + (KX)_{ij} - \dfrac{Y_{ij}}{(KA)_{ij}} \sum_{k=1}^{p} K_{ik}A_{kj} \ln \left( \dfrac{X_{kj}}{A_{kj}}(KA)_{ij}\right)$ \end{small}} \\
&\eqqcolon Q_{F, 1}(\myvec{X}, \myvec{A}).
\end{align*}
The condition $ Q_{F, 1}(\myvec{X}, \myvec{X}) = F(\myvec{X}) $ follows by simple algebraic manipulations, such that $Q_{F,1}$ is a valid surrogate functional for $ F. $\newline
The approach by Jensen's inequality is very flexible and we obtain different surrogate functionals $Q_{F,2}$ and $Q_{F,3}$ by using i.e.  $f_1(t)=Y_{ij} \ln(\nicefrac{Y_{ij}}{t}) - Y_{ij} +t$ or $f_2(t)=- Y_{ij} \ln(t) +t$ instead of $ f. $ Inserting the same $ \lambda_k $ and $ \alpha_k $ as before in Equation \eqref{eq:Jensensche Ungleichung:Konstruktion}, we obtain immediately the surrogates
\begin{align*}
\text{\begin{small} $Q_{F, 2}(\myvec{X}, \myvec{A})$ \end{small}} &=\text{\begin{small} $ \sum_{i=1}^{n} \sum_{j=1}^{m} \dfrac{1}{(KA)_{ij}} \sum_{k=1}^{p} K_{ik}A_{kj}  \left ( Y_{ij} \ln \left( \dfrac{Y_{ij}}{\dfrac{X_{kj}}{A_{kj}}(KA)_{ij}} \right) - Y_{ij} + \dfrac{X_{kj}}{A_{kj}} (KA)_{ij} \right ),$ \end{small}} \\
\text{\begin{small} $Q_{F, 3}(\myvec{X}, \myvec{A})$ \end{small}}&=\text{\begin{small} $\sum_{i=1}^{n} \sum_{j=1}^{m} \Bigg[ Y_{ij} \ln(Y_{ij}) - Y_{ij}$ \end{small}} \\
& \text{\begin{small} $\quad \quad \quad \quad \quad\quad+ \dfrac{1}{(KA)_{ij}} \sum_{k=1}^{p} K_{ik}A_{kj} \left( -Y_{ij} \ln \left( \dfrac{X_{kj}}{A_{kj}} (KA)_{ij} \right) \!+\! \dfrac{X_{kj}}{A_{kj}} (KA)_{ij} \right)\! \Bigg].$ \end{small}}
\end{align*}
It easy to check, that the partial derivatives for all three variants are the same, hence, the update rules obtained in the next section based on first order optimality conditions will be identical. Applying the same approach for obtaining a surrogate for $\myvec{K}$ yields the following theorem.

\begin{theorem}[Surrogate Functional for the Kullback-Leibler Divergence with Jensens Inequality] \label{satz:Surrogat-Funktionen für die Kullback-Leibler-Divergenz}
	We consider the cost functionals $ F(\myvec{X}) \coloneqq \KL(\myvec{Y}, \myvec{KX}) $ and $ G(\myvec{K}) \coloneqq \KL(\myvec{Y}, \myvec{KX}) $. Then
	\begin{align*}
	\text{\begin{small} $Q_F(\myvec{X}, \myvec{A})$ \end{small}} &\coloneqq  \text{\begin{small} $\sum_{i=1}^{n} \sum_{j=1}^{m} Y_{ij} \ln(Y_{ij}) - Y_{ij} + (KX)_{ij} - \dfrac{Y_{ij}}{(KA)_{ij}} \sum_{k=1}^{p} K_{ik}A_{kj} \ln \left( \dfrac{X_{kj}}{A_{kj}}(KA)_{ij}\right)$ \end{small}}\\
	\text{\begin{small} $Q_G(\myvec{K}, \myvec{A})$ \end{small}} & \coloneqq \text{\begin{small} $\sum_{i=1}^{n} \sum_{j=1}^{m} Y_{ij} \ln(Y_{ij}) - Y_{ij} + (KX)_{ij} - \dfrac{Y_{ij}}{(AX)_{ij}} \sum_{k=1}^{p} A_{ik}X_{kj} \ln \left( \dfrac{K_{ik}}{A_{ik}}(AX)_{ij}\right)$ \end{small}}
	\end{align*}
	define separable and convex surrogate functionals.
\end{theorem}

\subsection{Surrogates for $\ell_1$- and $ \ell_2$-Norm Penalties}\label{subsec:Surrogates for lp-penalty terms}
Computing an NMF is an ill-posed problem, see \cite{demol}, hence one needs to add stabilizing penalty terms for obtaining reliable matrix decompositions. The most standard penalties are $\ell_1$- and $ \ell_2 $-terms for the matrix factors leading to
\begin{equation} \label{eq:L1- und L2-Regularisierung:Minimumproblem}
\min_{\myvec{K}, \myvec{X}\geq 0} D_\beta (\myvec{Y}, \myvec{KX}) + \lambda \Vert \myvec{X} \Vert_1 + \dfrac{\mu}{2} \Vert \myvec{K} \Vert_F^2 + \dfrac{\nu}{2} \Vert \myvec{X} \Vert_F^2 + \omega \Vert \myvec{K} \Vert_1
\end{equation}
for $ \beta\in \{1, 2\}. $ \newline
The $\ell_2$-penalty prohibits exploding norms for each matrix factor and the $\ell_1$-term promotes sparsity in the minimizing factors, see \cite{bangti12,louis89} for a general exposition. Combinations of $\ell_1$- and $\ell_2$-norms are sometimes called elastic net regularization, \cite{Jin09}, due to there importance in medical imaging.\newline
These penalty terms are convex and they separate, hence, they can be used as surrogates themselves. For the case of Kullback-Leibler divergences this leads to the following surrogate for minimization with respect to $\myvec{X}$:

\begin{equation*}
Q_F(\myvec{X}, \myvec{A}) \coloneqq Q_{\KL}(\myvec{X}, \myvec{A}) + \lambda \Vert \myvec{X} \Vert_1 + \dfrac{\nu}{2} \Vert \myvec{X} \Vert_F^2,
\end{equation*}
where $ Q_{\KL} $ is the surrogate for the Kullback-Leibler divergence of Theorem \ref{satz:Surrogat-Funktionen für die Kullback-Leibler-Divergenz} for $ \myvec{X}. $\newline
The Frobenius case cannot be treated in the same way. If we use the penalty terms as surrogates themselves and obtain the standard minimization algorithm by first order optimality conditions, then this does not lead to a multiplicative algorithm, which preserves the non-negativity of the iterates. It can be easily seen that the $ \ell_1 $-penalty term causes this difficulty. For a more extended discussion on this, see Appendix \ref{app:subsec:Frobenius-norm}.\newline
Hence, we have to construct a different surrogate. Similar to the discussion in Subsection \ref{subsec:Frobenius discrepancy and LQBP}, we consider here $ f_{\myvec{y}}:\mathbb{R}_{\geq 0}^p \rightarrow \mathbb{R} $ with
\begin{equation*}
f_{\myvec{y}}(\myvec{x})\coloneqq \dfrac{1}{2} \Vert \myvec{y}-\myvec{Kx}\Vert^2 + \lambda \Vert \myvec{x} \Vert_1 + \dfrac{\nu}{2} \Vert \myvec{x}\Vert^2,
\end{equation*}
which yields the Hessian $\nabla^2f_{\myvec{y}}(\myvec{a})=\myvec{K}^\intercal\myvec{K} +\nu \myvec{I}.$ The choice of $\kappa_k$ is done dependent on the $ \ell_1 $ regularization term of the cost function $ f_{\myvec{y}} $ as already described in Subsection \ref{subsec:Frobenius discrepancy and LQBP}. It can be shown in the derivation of the NMF algorithm that $ \kappa_k = \lambda $ for all $ k $ leads to multiplicative update rules. A more general cost function is considered in Appendix \ref{app:subsec:Frobenius-norm}, where the concrete effect of $ \kappa_k $ is described in more detail.\newline
This yields the following diagonal matrix $ \myvec{\Lambda}_{f_{\myvec{y}}}(\myvec{a})$:

\begin{equation*}
\Lambda_{f_{\myvec{y}}}(\myvec{a})_{kk} = \dfrac{((K^\intercal K + \nu I_{p\times p} ) a )_k + \lambda}{a_k}.
\end{equation*}
The surrogate for minimization with respect to $\myvec{X}$ is then
\begin{equation*}
Q_{f_{\myvec{y}}}(\myvec{x}, \myvec{a}) = f_{\myvec{y}}(\myvec{a}) + \nabla f_{\myvec{y}}(\myvec{a})^\intercal (\myvec{x}-\myvec{a}) + \dfrac{1}{2} (\myvec{x}-\myvec{a})^\intercal \myvec{\Lambda}_{f_{\myvec{y}}}(\myvec{a}) (\myvec{x}-\myvec{a}).
\end{equation*}
Similar, for minimization with respect to $\myvec{K}$ we obtain the surrogate by using the diagonal matrix
\begin{equation*}
\Lambda_{g_{\myvec{y}}}(\myvec{a})_{kk} \coloneqq \dfrac{(a(XX^\intercal + \mu I_{p\times p}))_k + \omega}{a_k}.
\end{equation*}

\subsection{Surrogates for Orthogonality Constraints}\label{subsec:Surrogates for orthogonality constraints}
The observation that a non-negative matrix with pairwise orthogonal rows has at most one non-zero entry per column is the motivation for introducing orthogonality constraints for $\myvec{K}$ or $\myvec{X}$.
This will lead to strictly uncorrelated feature vectors, which is desirable in several applications e.g. for obtaining discriminating biomarkers from mass spectra, see Section \ref{sec:MALDI_Imaging} on MALDI Imaging.\newline
We could add the orthogonality  constraint $\myvec{K}^\intercal\myvec{K}=\myvec{I}$ as an additional penalty term $\sigma_{\myvec{K}} \|\myvec{K}^\intercal \myvec{K}-\myvec{I}\|^2$. However, this would introduce fourth order terms. Hence we introduce additional variables $\myvec{V}$ and $\myvec{W}$ and split the orthogonality condition into two second order terms leading to

\begin{equation} \label{eq:Orthogonale NMF:Minimumproblem (Hilfsvariablen)}
\begin{aligned}
\min_{\myvec{K}, \myvec{X}, \myvec{V}, \myvec{W}\geq 0} \Big\{ D_\beta (\myvec{Y}, \myvec{KX}) &+ \dfrac{\sigma_{\myvec{K}, 1}}{2} \Vert \myvec{I} - \myvec{V}^\intercal \myvec{K} \Vert_F^2 + \dfrac{\sigma_{\myvec{K}, 2}}{2} \Vert \myvec{V}-\myvec{K}\Vert_F^2 \\
& + \dfrac{\sigma_{\myvec{X}, 1}}{2} \Vert \myvec{I} - \myvec{X} \myvec{W}^\intercal \Vert_F^2 + \dfrac{\sigma_{\myvec{X}, 2}}{2} \Vert \myvec{W}-\myvec{X}\Vert_F^2 \Big\}.
\end{aligned}
\end{equation}
Surrogates for the terms $  \Vert \myvec{I} - \myvec{V}^\intercal \myvec{K} \Vert_F^2 $ and $ \Vert \myvec{I} - \myvec{X} \myvec{W}^\intercal \Vert_F^2 $ can be calculated via Jensen's inequality (see Subsection \ref{subsec: Frobenius_discrepancy_and_Jensens_inequality}). The other penalties can be used as surrogates themselves and therefore, we obtain
\begin{theorem}[Surrogate Functionals for Orthogonality Constraints] \label{satz:Surrogate functionals for orthogonal NMF}
	We consider the cost functionals
	\begin{align*}
	F(\myvec{X}) &\coloneqq \dfrac{\sigma_{\myvec{X}, 1}}{2} \Vert \myvec{I} - \myvec{X} \myvec{W}^\intercal \Vert_F^2 + \dfrac{\sigma_{\myvec{X}, 2}}{2} \Vert \myvec{W}-\myvec{X}\Vert_F^2 \eqqcolon G(\myvec{W}), \\
	H(\myvec{K}) &\coloneqq \dfrac{\sigma_{\myvec{K}, 1}}{2} \Vert \myvec{I} - \myvec{V}^\intercal \myvec{K} \Vert_F^2 + \dfrac{\sigma_{\myvec{K}, 2}}{2} \Vert \myvec{V}-\myvec{K}\Vert_F^2 \eqqcolon J(\myvec{V})
	\end{align*}
	with $ \sigma_{\myvec{X}, 1}, \sigma_{\myvec{X}, 2}, \sigma_{\myvec{K}, 1}, \sigma_{\myvec{K}, 2}\geq 0 $. Then 
	\begin{align*}
	\text{\begin{small} $Q_F(\myvec{X}, \myvec{A})$ \end{small}} &  \text{\begin{small} $\!\! \coloneqq\dfrac{\sigma_{\myvec{X}, 1}}{2} \sum_{k=1}^{p} \sum_{\ell=1}^{p} \dfrac{1}{(AW^\intercal)_{k \ell}} \sum_{j=1}^{m} A_{k j} W_{\ell j} \left ( \delta_{k \ell} - \dfrac{X_{k j}}{A_{k j}} (AW^\intercal)_{k \ell} \right )^2\! + \dfrac{\sigma_{\myvec{X}, 2}}{2} \Vert \myvec{W}\! - \! \myvec{X}\Vert_F^2$ \end{small}} \\
	\text{\begin{small} $Q_G(\myvec{W}, \myvec{A})$ \end{small}}& \text{\begin{small} $\!\! \coloneqq\dfrac{\sigma_{\myvec{X}, 1}}{2} \sum_{k=1}^{p} \sum_{\ell=1}^{p} \dfrac{1}{(XA^\intercal)_{k \ell}} \sum_{j=1}^{m} X_{k j} A_{\ell j} \left ( \delta_{k \ell} - \dfrac{W_{\ell j}}{A_{\ell j}} (XA^\intercal)_{k \ell} \right )^2\! + \dfrac{\sigma_{\myvec{X}, 2}}{2} \Vert \myvec{W}\! - \! \myvec{X}\Vert_F^2$ \end{small}} \\
	\text{\begin{small} $Q_H(\myvec{K}, \myvec{A})$ \end{small}} & \text{\begin{small} $\!\! \coloneqq\dfrac{\sigma_{\myvec{K}, 1}}{2} \sum_{k=1}^{p} \sum_{\ell=1}^{p} \dfrac{1}{(V^\intercal A)_{k \ell}} \sum_{i=1}^{n} V_{i k} A_{i \ell} \left ( \delta_{k \ell} - \dfrac{K_{i \ell}}{A_{i \ell}} (V^\intercal A)_{k \ell} \right )^2\! + \dfrac{\sigma_{\myvec{K}, 2}}{2} \Vert \myvec{V}\! -\! \myvec{K}\Vert_F^2$ \end{small}} \\
	\text{\begin{small} $Q_J(\myvec{V}, \myvec{A})$ \end{small}} & \text{\begin{small} $\!\! \coloneqq\dfrac{\sigma_{\myvec{K}, 1}}{2} \sum_{k=1}^{p} \sum_{\ell=1}^{p} \dfrac{1}{(A^\intercal K)_{k \ell}} \sum_{i=1}^{n} A_{i k} K_{i \ell} \left ( \delta_{k \ell} - \dfrac{V_{i k}}{A_{i k}} (A^\intercal K)_{k \ell} \right )^2\! + \dfrac{\sigma_{\myvec{K}, 2}}{2} \Vert \myvec{V}\! -\! \myvec{K}\Vert_F^2$ \end{small}}
	\end{align*}
	define separable and convex surrogate functionals.
\end{theorem}

\subsection{Surrogates for Total Variation Penalties}\label{subsec:Surrogates for TV-penalties}
Total variation (TV) penalty terms are the second important class of regularization terms besides $\ell_p$-penalty terms. TV-penalties aim at smooth or even piecewise constant minimizers, hence they are defined in terms of first order or higher order derivatives \cite{total2}.\newline
Originally, they were introduced for denoising applications in image processing \cite{rudin} but have since been applied to inpainting, deconvolution and other inverse problems, see e.g. \cite{total1}. The precise mathematical formulation of the total variation in the continuous case is described in the following definition.
\begin{definition}[Total Variation (Continuous)] \label{def:TV (kontinuierlich)}
	Let $ \Omega \subset \mathbb{R}^N $ be open and bounded. The \textbf{total variation} of a function $ u\in L^1_{\text{loc}}(\Omega ) $ is defined as
	\begin{equation*}%\label{eq:def:TV(kontinuierlich)}
	\TV(u)\coloneqq  \sup\left \{ -\int_{\Omega} u \operatorname{div} \phi \diff \myvec{x} \ : \ \phi\in C^\infty_c (\Omega, \mathbb{R}^N) \ \text{mit} \ \Vert \phi(\myvec{x}) \Vert \leq 1 \ \forall \myvec{x}\in \Omega \right \}.
	\end{equation*}
\end{definition} \vspace*{-1ex}
There exist several analytic relaxations of TV based on $\ell_1$-norms of the gradient, which are more tractable for analytical investigations. For numerical implementations one rather uses the $L_1$-norm of the gradient $\| \nabla f \|_{L_1}$ as a more computationally tractable substitute. For discretization the gradient is typically replaced by a finite difference approximation \cite{totaldiskret}. \newline
For applying TV-norms to data, we assume that the row index in the data matrix $\myvec{Y}$ refers to spatial locations and the column index to so-called channels. In this case, we consider the most frequently used isotropic TV for applying it to measured, discretized hyperspectral data.
\vspace*{-1ex}
\begin{definition}[Total Variation (Discrete)] \label{def:TV (diskret)}
	For fixed $\varepsilon_{\TV}>0,$ the \textbf{total variation} of a matrix  $\myvec{K}\in \mathbb{R}_{\geq 0}^{n\times p}$ is defined as 
	\begin{equation} \label{eq:def:TV (diskret)}
	\TV (\myvec{K}) \coloneqq \sum_{k=1}^p \psi_k \sum_{i=1}^n \sqrt{\varepsilon_{\TV}^2 + \sum_{\ell\in N_i} (K_{ik}-K_{\ell k})^2},
	\end{equation}
	where $\psi_k \in \mathbb{R}_{\geq 0}$ is a weighting of the $ k $-th data channel and $N_i \subseteq \{1, \dots, n\} \setminus \{i\}$ denotes the index set referring to spatially neighboring pixels.
\end{definition}
We will use the following short hand notation
\begin{equation} \label{eq:TV-Strafterm:Gradient}
\vert \nabla_{ik} \myvec{K} \vert \coloneqq \sqrt{\varepsilon_{\TV}^2 + \sum_{\ell\in N_i} (K_{ik}-K_{\ell k})^2},
\end{equation}
which can be seen as a finite difference approximation of the gradient magnitude of the image $ \myvec{K}_{\bullet, k} $ at pixel $ K_{ik}$ for some neighbourhood pixels defined by $ N_i. $ A typical choice for neighbourhood pixels in two dimensions for the pixel $ (0,0) $ is $ N_{(0,0)}\coloneqq \{(1,0), (0,1)\} $ to get an estimate of the gradient components along both axes. Finally, by introducing the positive constant $ \varepsilon_{\TV}>0, $ we get a differentiable approximation of the total variation penalty.\newline
%leading to the TV-penalty term \vspace*{-1ex}
%\begin{equation}%\label{eq:TV-Strafterm:TV umgeschrieben}
%\TV (\myvec{K}) = \sum_{k=1}^p \psi_k \sum_{i=1}^n \vert \nabla_{ik} \myvec{K} \vert .
%\end{equation}
%The computation of the gradient is approximated by a finite difference estimate and is based on the choice of the neighboring pixels $ N_i, $ which define 
%The gradient can be calculated by difference quotients of different order. This is reflected by the choice of $N_i$ (set of coefficients in the neighborhood of $(i,k)$ and the choice of weights $\psi_k$).
In Section \ref{sec:MALDI_Imaging}, we will discuss the application of NMF methods to hyperspectral MALDI imaging datasets, which has a natural 'spatial structure' in its columns. \newline
In this section we stay with a generic choice of $N_i$ as well as of the $\psi_k$ and we construct a surrogate following the approach of the groundbreaking works of \cite{defrise} and \cite{oliveira}.\newline
For $t\geq 0$ and $s>0$ we use the inequality
(linear majorization)
\begin{equation} \label{eq:TV-Strafterm:Lineare Majorisierung}
\sqrt{t} \leq \sqrt{s} + \dfrac{t-s}{2\sqrt{s}}
\end{equation}
and apply it in order to compare $\nabla_{ik}\myvec{K}$ with values obtained by an arbitrary non-negative matrix $\myvec{A}$:

\begin{align*}
\vert \nabla_{ik} \myvec{K} \vert &\leq \vert \nabla_{ik} \myvec{A} \vert + \dfrac{\vert \nabla_{ik} \myvec{K} \vert^2 - \vert \nabla_{ik} \myvec{A} \vert^2}{2\vert \nabla_{ik} \myvec{A}\vert} \\
&\leq \dfrac{\vert \nabla_{ik} \myvec{K} \vert^2 + \vert \nabla_{ik} \myvec{A} \vert^2}{2\vert \nabla_{ik} \myvec{A} \vert} \\
&= \dfrac{2\varepsilon_{\TV}^2 + \sum_{\ell\in N_i} \left [(K_{ik}-K_{\ell k})^2 + (A_{ik}-A_{\ell k})^2\right ]}{2\vert \nabla_{ik} \myvec{A} \vert}.
\end{align*}
Summation with respect to $ i, $ multiplication with $ \psi_k $ and summation with respect to $ k $ leads to
\begin{equation*}
\TV(\myvec{K}) \leq \sum_{k=1}^p \psi_k \sum_{i=1}^n \dfrac{2\varepsilon_{\TV}^2 + \sum_{\ell\in N_i} \left [(K_{ik}-K_{\ell k})^2 + (A_{ik}-A_{\ell k})^2\right ]}{2\vert \nabla_{ik} \myvec{A} \vert} \eqqcolon Q_{\TV}^{\text{Oli}}(\myvec{K}, \myvec{A}).
\end{equation*}
This yields a candidate for a surrogate $Q_{\text{TV}}^{\text{Oli}}$ for the TV-penalty term, which is the same as the one used in \cite{oliveira}. However, it is not separable, hence we aim at a second, separable approximation. For arbitrary $a,b,c,d \in \mathbb{R}$ we have
\begin{equation} \label{eq:rem:TV-Strafterm:Ungleichung mit reellen Zahlen}
\dfrac{1}{2}\left ((a-b)^2 + (c-d)^2 \right ) \leq (a-b)(c-d) + (b-d)^2 + (a-c)^2.
\end{equation}
This leads to 
\begin{align*}
Q_{\TV}^{\text{Oli}}(\myvec{K}, \myvec{A}) & \text{\begin{small} $= \sum_{k=1}^p \psi_k \sum_{i=1}^n \dfrac{\varepsilon_{\TV}^2 + \sum_{\ell \in N_i} \nicefrac{1}{2} \left [(K_{ik}-K_{\ell k})^2 + (A_{ik}-A_{\ell k})^2\right ]}{\vert \nabla_{ik} \myvec{A} \vert}$ \end{small}}\\
&\text{\begin{scriptsize} $\leq \sum_{k=1}^p \psi_k \! \sum_{i=1}^n \dfrac{\varepsilon_{\TV}^2 \! +\! \sum_{\ell \in N_i} \! \left [ (K_{ik}\! - \! K_{\ell k}) (A_{ik} \! - \! A_{\ell k}) \! +\! (K_{\ell k} \! - \! A_{\ell k})^2 \! +\! (K_{ik} \! - \! A_{ik})^2 \right ]}{\vert \nabla_{ik} \myvec{A} \vert}$ \end{scriptsize}} \\
&\eqqcolon Q_{\TV}(\myvec{K}, \myvec{A}).
\end{align*}
Therefore, we have the following

\begin{theorem}[Surrogate Functional for TV Penalty Term] \label{satz:Surrogat-Funktion zum TV-Strafterm}
	We consider the cost functional $F(\myvec{K})\coloneqq \TV(\myvec{K}) $ with the total variation defined in \eqref{eq:def:TV (diskret)}. Then
	\begin{equation*}
	\text{ \begin{footnotesize} $ \!\! Q_{\TV}(\myvec{K}, \! \myvec{A})\! \coloneqq\!\! \sum_{k=1}^p\! \psi_k \! \sum_{i=1}^n \! \dfrac{\varepsilon_{\TV}^2 \! +\! \sum_{\ell \in N_i} \! \left [ (K_{ik}\! - \! K_{\ell k}) (A_{ik} \! - \! A_{\ell k}) \! +\! (K_{\ell k} \! - \! A_{\ell k})^2 \! +\! (K_{ik} \! - \! A_{ik})^2 \right ]}{\vert \nabla_{ik} A \vert}$ \end{footnotesize} }
	\end{equation*}
	defines a separable and convex surrogate functional.
\end{theorem}
The separability of the surrogate is not obvious. The proof (see Appendix \ref{app:sec:TV_Separability}) delivers the following notation, which we also need for an description of the algorithms in the next section. First of all we need the definition of the so-called adjoint neighborhood pixels $ \bar{N}_i $ given by
\begin{equation}\label{eq:TV-Strafterm:Äquivalenz Nachbarn}
\ell \in \bar{N}_i \Leftrightarrow  i \in N_\ell.
\end{equation}
One then introduces matrices $ P(\myvec{A})\in \mathbb{R}^{n\times p}_{\geq 0}$ and $Z(\myvec{A})\in \mathbb{R}^{n\times p}_{\geq 0}$ via

\begin{equation}\label{eq:TV-Strafterm:P}
P(\myvec{A})_{i k} \coloneqq  \dfrac{1}{\vert \nabla_{ik} \myvec{A} \vert} \sum_{\ell \in N_i} 1 + \sum_{\ell \in \bar{N}_i} \dfrac{1}{\vert \nabla_{\ell k} \myvec{A} \vert},
\end{equation}

\begin{equation}\label{eq:TV-Strafterm:Z}
Z(\myvec{A})_{i k} \coloneqq  \dfrac{1}{P(\myvec{A})_{i k}} \left ( \dfrac{1}{\vert \nabla_{ik} \myvec{A} \vert} \sum_{\ell \in {N}_i} \dfrac{A_{i k} + A_{\ell  k}}{2} + \sum_{\ell \in \bar{N}_i} \dfrac{A_{i k} + A_{\ell  k}}{2 \vert \nabla_{\ell k} \myvec{A} \vert} \right ).
\end{equation}
Using these notations, it can be shown that the surrogate can be written as
\begin{equation}\label{eq:TV-Strafterm:Separabel}
	Q_{\TV}(\vec{K}, \vec{A}) = \sum_{k=1}^p \psi_k \sum_{i=1}^n \left [ P(\myvec{A})_{ik} (K_{ik} - Z(\myvec{A})_{ik})^2  \right ] + C(\myvec{A}),
\end{equation}
such that we obtain the desired separability. Here, $ C(\myvec{A}) $ denotes some function depending on $ \myvec{A}. $ The description of $ Q_{\TV} $ with the help of $ P(\myvec{A})_{ik}$ and $ Z(\myvec{A})_{i k} $ will also allow us to compute the partial derivatives in a way more comfortable way (see also Appendix \ref{app:subsec:KLD}).
\subsection{Surrogates for Supervised NMF}\label{subsec:Surrogates for supervised NMF}
As a motivation for this section, we consider classification tasks. We view the data matrix $\myvec{Y}$ as a collection of $n$ data vectors, which are stored in the rows of $\myvec{Y}$. Moreover, we do have an expert annotation $u_i$ for $ i=1,\dots, n,$ which assigns a label to each data vector. For a classification problem with two classes we have $u_i \in \{0,1\}$.\newline
As already stated, the rows of the matrix $\myvec{X}$ of an NMF decomposition can be regarded as a basis for approximating the rows of $\myvec{Y}$. Hence, one assumes that the correlations between a row $\myvec{Y}_{i,\bullet}$ of $\myvec{Y}$ and all row vectors of $\myvec{X}$, i.e. computing $ \myvec{Y}_{i,\bullet}\myvec{X}^\intercal$, contains the relevant information of $\myvec{Y}_{i,\bullet}$.
The vector of correlations yields a so-called feature vector of length $p$.
A classical linear regression model, which uses these feature vectors, then asks to compute weights $\beta_k$ for $ k=1,\dots,p,$ such that $\myvec{Y}_{i,\bullet} \myvec{X}^\intercal\myvec{\beta} \approx u_i$ (for more details on linear discriminant analysis methods, we refer to Chapter 4 in \cite{bishop06}).\newline
In matrix notation and using least squares, this is equivalent to computing $\myvec{\beta}$ as a minimizer of $$\|\myvec{u}- \myvec{Y}\myvec{X}^\intercal \myvec{\beta} \|^2.$$
We now use  $\myvec{X}$ and $\myvec{\beta}$ to define 
$$ \myvec{x}^* \coloneqq \myvec{X}^\intercal \myvec{\beta}. $$
In tumor typing classifications, where the data matrix $ \myvec{Y} $ is obtained by MALDI measurements, the vector $ \myvec{x}^* $ can be interpreted as a characteristic mass spectrum of some specific tumor type and can be directly used for classification tasks in the arising field of digital pathology (see also Section \ref{sec:MALDI_Imaging} and \cite{leuschner18}).\newline
The classification of a new data vector $\myvec{y}$ is then simply obtained by computing the scalar product  $w={\myvec{x}^*}^\intercal \myvec{y}$ and assigning either the class label $0$ or $1$ by comparing $w$ with a pre-assigned threshold $s$. This threshold is typically obtained in the training phase of the classification procedure by computing $ \myvec{YX}^\intercal\myvec{\beta} $ for some given training data $ \myvec{Y} $ and choosing $ s, $ such that a performance measure of the classifier is optimized. \newline
The approach we have described is based on first computing an NMF, i.e. $\myvec{K}$ and $\myvec{X}$, and then computing the weights $\myvec{\beta}$ of the classifier. Hence, the computation of the NMF is done independently of the class labels $\myvec{u}$, which is also referred to as an unsupervised NMF approach. We might expect, that computing the NMF by minimizing a functional which includes the class labels, i.e.
\begin{equation*}
	F(\myvec{K}, \myvec{X}, \myvec{\beta}) \coloneqq D_\beta (\myvec{Y}, \myvec{KX}) + {\frac{\rho}{2}} 
	\|\myvec{u} - \myvec{Y}\myvec{X}^\intercal \myvec{\beta} \|^2,
\end{equation*}
will lead to an improved classifier. In the application field of MALDI imaging, this supervised approach yields an extraction of features from the given training data, which allow a better distinction between spectra obtained from different tissue types such as tumorous and non-tumorous regions (see also \cite{leuschner18}).\newline
Surrogates for the first term have been determined in the previous section for the case of the Frobenius norm and the Kullback-Leibler divergence. Hence, we need to determine surrogates of the new penalty term for minimization with respect to $\myvec{X}$ and $\myvec{\beta}$:
\begin{align}
F(\myvec{X}) &\coloneqq \dfrac{1}{2} \Vert \myvec{u} - \myvec{YX}^\intercal \myvec{\beta} \Vert^2,  \label{eq:Surrogat zur linearen Regression:F}\\
G(\myvec{\beta}) &\coloneqq \dfrac{1}{2} \Vert \myvec{u} - \myvec{YX}^\intercal \myvec{\beta} \Vert^2. \label{eq:Surrogat zur linearen Regression:G}
\end{align}
Surrogates can be obtained by extending the Jensen principle to the matrix valued case. Here, we consider a convex subset
$ \Omega \subset \mathbb{R}^{N\times M}_{>0} $ and define
\begin{equation}\label{eq:Surrogat zur linearen Regression:Ansatz}
\begin{aligned}
\tilde{F}: \Omega &\rightarrow \mathbb{R} \\
\myvec{V}	&\mapsto f(\myvec{c}^\intercal \myvec{V} \myvec{d})
\end{aligned}
\end{equation}
with a convex and continuously differentiable function $f$ and auxiliary variables $ \myvec{c}\in \mathbb{R}^N_{>0} $ and $ \myvec{d}\in \mathbb{R}^M_{>0}. $
We now use the following generalized Jensen's inequality
\begin{equation} \label{eq:Surrogat zur linearen Regression:Jensensche Ungleichung}
f\left ( \sum_{j=1}^{N} \sum_{k=1}^{M} \lambda_{j k} \alpha_{jk} \right ) \leq  \sum_{j=1}^{N} \sum_{k=1}^{M} \lambda_{jk} f(\alpha_{jk}).
\end{equation}
Setting
\begin{align} 
\lambda_{j k} &= \dfrac{Y_{ij} A_{k j} \beta_k }{\myvec{Y}_{i, \bullet} \myvec{A}^\intercal \myvec{\beta} } \label{eq:Surrogat zur linearen Regression:lambda_i_ell:final} \\
\alpha_{j k} &= \dfrac{Y_{ij} X_{k j} \beta_k }{\lambda_{j k}}. \label{eq:Surrogat zur linearen Regression:alpha_i_ell:final}
\end{align}
for some $ i\in \{1, \dots,n \} $ yields by inserting $ \lambda_{j k} $ and $ \alpha_{j k} $ into $ \eqref{eq:Surrogat zur linearen Regression:Jensensche Ungleichung} $
\begin{equation*}
F(\myvec{X}) \leq \dfrac{1}{2} \sum_{i=1}^{n} \dfrac{1}{(YA^\intercal \beta)_i} \sum_{j=1}^{m} \sum_{k=1}^{p} Y_{ij} A_{k j} \beta_k \left ( u_i -  \dfrac{X_{kj}}{A_{kj}} (YA^\intercal \beta)_i \right )^2 \coloneqq Q_F(\myvec{X}, \myvec{A}).
\end{equation*}
The computation of a surrogate for minimization with respect to $\myvec{\beta}$ proceeds analogously. We summarize the results in the following theorem.

\begin{theorem}[Surrogate Functionals for Linear Regression] \label{satz:Surrogat-Funktionen zur linearen Regression (SPS)}
	Let $ F(\myvec{X}) \coloneqq  \nicefrac{1}{2}\Vert \myvec{u} - \myvec{YX}^\intercal \myvec{\beta} \Vert^2 $ und $ G(\myvec{\beta}) \coloneqq  \nicefrac{1}{2}\Vert \myvec{u} - \myvec{YX}^\intercal \myvec{\beta} \Vert^2 $ denote a cost functional with repect to $\myvec{X}$ and $\myvec{\beta}$. Then
	\begin{align}
	Q_F(\myvec{X}, \myvec{A}) &\coloneqq \dfrac{1}{2} \sum_{i=1}^{n} \dfrac{1}{(YA^\intercal \beta)_i} \sum_{j=1}^{m} \sum_{k=1}^{p} Y_{ij} A_{k j} \beta_k \left ( u_i -  \dfrac{X_{kj}}{A_{kj}} (YA^\intercal \beta)_i \right )^2 \label{eq:satz:Surrogat-Funktionen zur linearen Regression (SPS):X}\\
	Q_G(\myvec{\beta}, \myvec{a}) &\coloneqq \dfrac{1}{2} \sum_{i=1}^{n} \dfrac{1}{(YX^\intercal a)_i} \sum_{k=1}^{p} (YX^\intercal)_{ik} a_k \left( u_i - \dfrac{\beta_k}{a_k} (YX^\intercal a)_i \right)^2 \label{eq:satz:Surrogat-Funktionen zur linearen Regression (SPS):K}
	\end{align}
	define separable and convex surrogate functionals.
\end{theorem}
A big advantage of linear regression models are their simplicity and manageability. However, they are by far not the optimal approach to approximate the binary output data $ \myvec{u} $ with a continuous input. Logistic regression models offer a way more natural method for binary classification tasks. Together with the supervised NMF as a feature extraction method, this overall workflow leads in \cite{leuschner18} to excellent classification results and outperformed classical approaches.\newline
However, the proposed model is based on a gradient descent approach, such that the non-negativity of the iterates can only be guaranteed by a projection step. Appropriate surrogate functionals for this workflow are still ongoing research and could lead to even better outcomes (see also \cite{ZKY05,zhang}).
\section{Surrogate Based NMF Algorithms} \label{sec:Surrogate_based_NMF_algorithms}
In the previous section we have defined surrogate functionals for various NMF cost functions. Besides the necessary surrogate properties we also expect that the minimization of these surrogates is straightforward and can be computed efficiently.\newline
In our case we demand additionally, that the minimization schemes based on solving the first order optimality conditions leads to a separable algorithm and that it only requires multiplicative updates, which automatically preserve the non-negativity of its iterates.
Let us start with denoting the most general functional with Kullback-Leibler divergence, the Frobenius case follows similarly.\newline
For constructing non-negative matrix factorizations, we incorporate $\ell_2$-, sparsity-, orthogonality-, TV-constraints and also the penalty terms coming from the supervised NMF. Of course, in most applications one only uses a subset of these constraints for stabilization and for enhancing certain properties. These algorithms can readily be obtained from the general case by putting the respective regularization parameters to zero. The corresponding update rules are classical results and can be found in numerous works \cite{demol,lecharlier13,leeseung}.
\begin{definition}[NMF Problem] \label{def:NMF-Problem_KLD}
	For $\myvec{Y}\in \mathbb{R}_{\geq 0}^{n\times m}, \ \myvec{K}, \myvec{V} \in \mathbb{R}_{\geq 0}^{n\times p}, \ \myvec{X}, \myvec{W} \in \mathbb{R}_{\geq 0}^{p\times m} $, $ \myvec{\beta}\in \mathbb{R}_{\geq 0}^{p} $ 
	and a set of regularization parameters $ \lambda, \mu, \nu, \omega, \tau, \sigma_{\myvec{K}, 1}, \sigma_{\myvec{K}, 2}, \sigma_{\myvec{X}, 1}, \linebreak \sigma_{\myvec{X}, 2}, \rho \geq 0, $ 
	we define the NMF minimization problem by
	\begin{align*}
	\min_{\myvec{K}, \myvec{X}, \myvec{V}, \myvec{W}, \myvec{\beta} \geq 0}
	& \bigg\{\KL(\myvec{Y}, \myvec{KX}) + \lambda {\Vert \myvec{X} \Vert}_1 + \frac{\mu}{2} {\Vert \myvec{K} \Vert}_F^2 + \frac{\nu}{2} {\Vert \myvec{X} \Vert}_F^2 + \omega {\Vert \myvec{K} \Vert}_1 + \dfrac{\tau}{2} \TV(\myvec{K}) \\
	&+ \dfrac{\sigma_{\myvec{K}, 1}}{2} \Vert \myvec{I} - \myvec{V}^\intercal \myvec{K} \Vert_F^2 + \dfrac{\sigma_{\myvec{K}, 2}}{2} \Vert \myvec{V}-\myvec{K}\Vert_F^2 + \dfrac{\sigma_{\myvec{X}, 1}}{2} \Vert \myvec{I} - \myvec{X} \myvec{W}^\intercal \Vert_F^2 \\
	&+ \dfrac{\sigma_{\myvec{X}, 2}}{2} \Vert \myvec{W}-\myvec{X}\Vert_F^2 + \dfrac{\rho}{2} \Vert \myvec{u} - \myvec{YX}^\intercal \myvec{\beta} \Vert^2 \bigg\}.
	\end{align*}
\end{definition}%
The choice of the various regularization parameters occuring in Definition \ref{def:NMF-Problem_KLD} is often based on heuristic approaches. We will not focus on that issue in this work and refer instead to \cite{IJT11} and the references therein, where two methods are investigated for the general case of multi-parameter Tikhonov regularization.\newline
The algorithms studied in this section will start with positive initializations for $\myvec{K},\myvec{X},\myvec{V},\myvec{W}$ and $ \myvec{\beta}. $ These matrices are updated alternatingly, i.e. all matrices except one matrix are kept fixed and only the selected matrix is updated by solving the respective first order optimality condition.\newline
We will focus in this section on the derivation of the update rules of $\myvec{K}$ (see also Appendix \ref{app:subsec:KLD}). The iteration schemes for the other matrices follow analogously.\newline
For that, we only have to consider those terms in the general functional which depend on $\myvec{K}$, i.e. we aim at determining a minimizer for
\begin{equation*}
\text{\begin{small} $F(\myvec{K}) \coloneqq \KL(\myvec{Y}, \myvec{KX}) + \frac{\mu}{2} {\Vert \myvec{K} \Vert}_F^2 + \omega {\Vert \myvec{K} \Vert}_1 + \dfrac{\tau}{2} \TV(\myvec{K}) + \dfrac{\sigma_{\myvec{K}, 1}}{2} \Vert \myvec{I}\! -\! \myvec{V}^\intercal \myvec{K} \Vert_F^2 + \dfrac{\sigma_{\myvec{K}, 2}}{2} \Vert \myvec{V}\!-\!\myvec{K}\Vert_F^2.$ \end{small}}
\end{equation*}
Instead of minimizing this functional we exchange it with the previously constructed surrogate functionals, which leads to

\begin{equation*}
Q_F(\myvec{K}, \myvec{A}) \coloneqq Q_{\text{KL}}(\myvec{K}, \myvec{A}) + \frac{\mu}{2} {\Vert \myvec{K} \Vert}_F^2 + \omega {\Vert \myvec{K} \Vert}_1 + \dfrac{\tau}{2} Q_{\text{TV}}(\myvec{K}, \myvec{A}) + Q_{\text{Orth}}(\myvec{K}, \myvec{A})
\end{equation*}
with the surrogates $ Q_{\text{KL}} $ for the Kullback-Leibler divergence in Theorem \ref{satz:Surrogat-Funktionen für die Kullback-Leibler-Divergenz}, $ Q_{\text{TV}} $ for the TV penalty term in Theorem \ref{satz:Surrogat-Funktion zum TV-Strafterm} and $ Q_{\text{Orth}} $ for the orthogonality penalty terms in Theorem \ref{satz:Surrogate functionals for orthogonal NMF}.\newline
The computation of the partial derivatives leads to a system of equations
%	\begin{equation*} 
\begin{equation} \label{eq:NMF-Problem mit KLD:PartielleAbleitungQF}
\dfrac{\partial Q_F}{\partial K_{\xi \zeta}}(\myvec{K}, \myvec{A}) =0
\end{equation}
for $ \xi\in \{1, \dots, n\} $ and $ \zeta\in \{1, \dots, p\} $. This leads to a system of quadratic equations
\begin{align*}
&K_{\xi \zeta}^2 \left (\mu + \tau \psi_\zeta P_{\xi \zeta}(\myvec{A}) + \dfrac{\sigma_{\myvec{K}, 1}}{A_{\xi \zeta}} (VV^\intercal A)_{\xi \zeta} + \sigma_{\myvec{K}, 2}\right ) \\
+&K_{\xi \zeta} \left ( \sum_{j=1}^{m} X_{\zeta j} + \omega - \tau \psi_\zeta P(\myvec{A})_{\xi \zeta} Z(\myvec{A})_{\xi \zeta} - (\sigma_{\myvec{K}, 1} + \sigma_{\myvec{K}, 1}) V_{\xi \zeta}  \right ) \\
=&A_{\xi \zeta} \sum_{j=1}^{m} \dfrac{Y_{\xi j}}{(AX)_{\xi j}} X_{\zeta j}.
\end{align*}
Solving for $ K_{\xi \zeta} $ and denoting the Hadamard product by $\circ$ as well as the matrix division for each entry separately by a fraction line, yields the following update rule for $\myvec{K}.$ (Note that the notation for $P(\myvec{A})$ and $Z(\myvec{A})$ was introduced in the section on TV-regularization above.)
\begin{align*}
\text{\begin{small} $\myvec{K}^{[d+1]}$ \end{small}} &\text{\begin{small} $=\left [ \underbrace{\left ( \dfrac{\myvec{K}^{[d]}}{\mu \myvec{1}_{n\times p} + \tau \myvec{\Psi} \circ P(\myvec{K}^{[d]}) + \sigma_{\myvec{K}, 1} \dfrac{\myvec{VV}^\intercal \myvec{K}^{[d]}}{\myvec{K}^{[d]}} + \sigma_{\myvec{K}, 2} \myvec{1}_{n\times p} } \right ) \circ \left ( \dfrac{\myvec{Y}}{\myvec{K}^{[d]}\myvec{X}}\myvec{X}^\intercal \right )}_{\eqqcolon \myvec{\Theta}^{[d]}} \right .$ \end{small}} \\
&\text{\begin{small} $\left . +\dfrac{1}{4} \left ( \dfrac{\myvec{1}_{n\times m}\myvec{X}^\intercal + \omega\myvec{1}_{n\times p} - \tau \myvec{\Psi} \circ P(\myvec{K}^{[d]}) \circ Z(\myvec{K}^{[d]}) - (\sigma_{\myvec{K}, 1} + \sigma_{\myvec{K}, 2}) \myvec{V} }{\mu \myvec{1}_{n\times p} + \tau \myvec{\Psi} \circ P(\myvec{K}^{[d]}) + \sigma_{\myvec{K}, 1} \dfrac{\myvec{VV}^\intercal \myvec{K}^{[d]}}{\myvec{K}^{[d]}} + \sigma_{\myvec{K}, 2} \myvec{1}_{n\times p}} \right )^2 \ \right ]^{\nicefrac{1}{2}}$ \end{small}} \\
&\text{\begin{small} $-\dfrac{1}{2} \underbrace{\left ( \dfrac{\myvec{1}_{n\times m}\myvec{X}^\intercal + \omega\myvec{1}_{n\times p} - \tau \myvec{\Psi} \circ P(\myvec{K}^{[d]}) \circ Z(\myvec{K}^{[d]}) - (\sigma_{\myvec{K}, 1} + \sigma_{\myvec{K}, 2}) \myvec{V} }{\mu \myvec{1}_{n\times p} + \tau \myvec{\Psi} \circ P(\myvec{K}^{[d]}) + \sigma_{\myvec{K}, 1} \dfrac{\myvec{VV}^\intercal \myvec{K}^{[d]}}{\myvec{K}^{[d]}} + \sigma_{\myvec{K}, 2} \myvec{1}_{n\times p}} \right )}_{\eqqcolon \myvec{\Phi}^{[d]}}$ \end{small}}
\end{align*}
In the above update rule, $ \myvec{1}_{n\times p} $ denotes an $ n\times p $ matrix with ones in every entry and $ \myvec{\Psi}\in \mathbb{R}^{n\times p}_{\geq 0} $ is defined as
\begin{equation*}
	\myvec{\Psi} \coloneqq  
	\begin{pmatrix}
	\psi_1 & \psi_2 & \cdots  & \psi_p \\
	\psi_1 & \psi_2 & \cdots  & \psi_p \\
	\vdots & \vdots &  & \vdots \\
	\psi_1 & \psi_2  & \cdots  & \psi_p  \\
	\end{pmatrix}.
\end{equation*}
Details on the derivation can be found in Appendix \ref{app:subsec:KLD}.\newline
The partial derivatives with repect to $\myvec{X}$ are computed similarly. Defining

\begin{align*}
\text{\begin{small} $\myvec{\Lambda}^{[d]}$ \end{small}}&\text{\begin{small} $\coloneqq \left ( \dfrac{\myvec{X}^{[d]}}{\nu \myvec{1}_{p\times m} + \sigma_{\myvec{X}, 1} \dfrac{\myvec{X}^{[d]}\myvec{W}^\intercal \myvec{W}}{\myvec{X}^{[d]}} + \sigma_{\myvec{X}, 2} \myvec{1}_{p\times m} + \rho \dfrac{\myvec{\beta\beta}^\intercal \myvec{X}^{[d]}\myvec{Y}^\intercal \myvec{Y}}{\myvec{X}^{[d]}} } \right ) \circ \left ( \myvec{K}^\intercal \dfrac{\myvec{Y}}{\myvec{K} \myvec{X}^{[d]}} \right ) $ \end{small}}\\
\myvec{\Gamma}^{[d]} &\coloneqq \dfrac{\myvec{K}^\intercal \myvec{1}_{n\times m} + \lambda \myvec{1}_{p\times m} - (\sigma_{\myvec{X}, 1} + \sigma_{\myvec{X}, 2})\myvec{W} - \rho \myvec{\beta u}^\intercal \myvec{Y} }{\nu \myvec{1}_{p\times m} + \sigma_{\myvec{X}, 1} \dfrac{\myvec{X}^{[d]}\myvec{W}^\intercal \myvec{W}}{\myvec{X}^{[d]}} + \sigma_{\myvec{X}, 2} \myvec{1}_{p\times m} + \rho \dfrac{\myvec{\beta\beta}^\intercal \myvec{X}^{[d]}\myvec{Y}^\intercal \myvec{Y}}{\myvec{X}^{[d]}} }
\end{align*}
leads to the update
\begin{equation*}
\myvec{X}^{[d+1]} = \sqrt{\myvec{\Lambda}^{[d]} + \dfrac{1}{4} \myvec{\Gamma}^{[d]} \circ \myvec{\Gamma}^{[d]}} - \dfrac{1}{2} \myvec{\Gamma}^{[d]}.
\end{equation*}
The updates for $\myvec{V}$, $\myvec{W}$ are straight forward and we obtain the following theorem.

\begin{theorem}[Alternating Algorithm for NMF Problem in Definition \ref{def:NMF-Problem_KLD}] \label{satz:Alternierender Algorithmus für das NMF-Problem}
	The initializations $\myvec{K}^{[0]}, \myvec{V}^{[0]} \in \mathbb{R}_{> 0}^{n\times p},\ \myvec{X}^{[0]}, \myvec{W}^{[0]} \in \mathbb{R}_{> 0}^{p\times m}, \myvec{\beta}^{[0]}\in \mathbb{R}_{> 0}^{p} $ and the iterative updates
	\begingroup
	\allowdisplaybreaks
	\begin{align}
	\myvec{V}^{[d+1]} &= \dfrac{(\sigma_{\myvec{K}, 1} + \sigma_{\myvec{K}, 2}) \myvec{K}^{[d]}}{\sigma_{\myvec{K}, 1} \dfrac{\myvec{K}^{[d]} {\myvec{K}^{[d]}}^\intercal \myvec{V}^{[d]} }{\myvec{V}^{[d]}} + \sigma_{\myvec{K}, 2} \myvec{1}_{n\times p}} \label{eq:satz:Alternierender Algorithmus für das NMF-Problem:UpdateV}\\
	\myvec{K}^{[d+1]} &= \sqrt{\myvec{\Theta}^{[d]} + \dfrac{1}{4} \myvec{\Phi}^{[d]} \circ \myvec{\Phi}^{[d]}} - \dfrac{1}{2} \myvec{\Phi}^{[d]} \label{eq:satz:Alternierender Algorithmus für das NMF-Problem:UpdateK}\\
	\myvec{W}^{[d+1]} &=\dfrac{(\sigma_{\myvec{X}, 1} + \sigma_{\myvec{X}, 2}) \myvec{X}^{[d]}}{\sigma_{\myvec{X}, 1} \dfrac{\myvec{W}^{[d]} {\myvec{X}^{[d]}}^\intercal \myvec{X}^{[d]} }{\myvec{W}^{[d]}} + \sigma_{\myvec{X}, 2} \myvec{1}_{p\times m}} \label{eq:satz:Alternierender Algorithmus für das NMF-Problem:UpdateW} \\
	\myvec{X}^{[d+1]} &= \sqrt{\myvec{\Lambda}^{[d]} + \dfrac{1}{4} \myvec{\Gamma}^{[d]} \circ \myvec{\Gamma}^{[d]}} - \dfrac{1}{2} \myvec{\Gamma}^{[d]} \label{eq:satz:Alternierender Algorithmus für das NMF-Problem:UpdateX}\\
	\myvec{\beta}^{[d+1]} &= \dfrac{\myvec{X}^{[d+1]}\myvec{Y}^\intercal \myvec{u}}{\myvec{X}^{[d+1]}\myvec{Y}^\intercal \myvec{Y}{\myvec{X}^{[d+1]}}^\intercal \myvec{\beta}^{[d]}} \circ \myvec{\beta}^{[d]} \label{eq:satz:Alternierender Algorithmus für das NMF-Problem:UpdateBeta}
	\end{align}
	\endgroup
	lead to a monotonically decrease of the cost functional
	\begin{align*}
	F(\myvec{K}, \myvec{X}, \myvec{V}, \myvec{W}, \myvec{\beta}) &\coloneqq \KL(\myvec{Y}, \myvec{KX}) + \lambda {\Vert \myvec{X} \Vert}_1 + \frac{\mu}{2} {\Vert \myvec{K} \Vert}_F^2 + \frac{\nu}{2} {\Vert \myvec{X} \Vert}_F^2 + \omega {\Vert \myvec{K} \Vert}_1 \\
	&+ \dfrac{\tau}{2} \TV(\myvec{K}) + \dfrac{\sigma_{K, 1}}{2} \Vert \myvec{I} - \myvec{V}^\intercal \myvec{K} \Vert_F^2 + \dfrac{\sigma_{K, 2}}{2} \Vert \myvec{V}-\myvec{K}\Vert_F^2 \\
	&+ \dfrac{\sigma_{X, 1}}{2} \Vert \myvec{I}- \myvec{X} \myvec{W}^\intercal \Vert_F^2 + \dfrac{\sigma_{X, 2}}{2} \Vert \myvec{W}-\myvec{X}\Vert_F^2 + \dfrac{\rho}{2} \Vert \myvec{u} - \myvec{YX}^\intercal \myvec{\beta} \Vert^2.
	\end{align*}
\end{theorem}
It is easy to see that the classical, regularized NMF algorithms described in \cite{demol,lecharlier13,leeseung} can be regained by putting the corresponding regularization parameters to zero. In the case of $ \ell_1 $- and $ \ell_2 $-regularized NMF, this leads to the cost function
\begin{equation*}
F(\myvec{K}, \myvec{X}) = \KL(\myvec{Y}, \myvec{KX}) + \lambda {\Vert \myvec{X} \Vert}_1 + \frac{\mu}{2} {\Vert \myvec{K} \Vert}_F^2 + \frac{\nu}{2} {\Vert \myvec{X} \Vert}_F^2 + \omega {\Vert \myvec{K} \Vert}_1.
\end{equation*}
The classical update rule for $ \myvec{X} $ is obtained by setting
\begin{align*}
\tilde{\myvec{\Lambda}}^{[d]} &\coloneqq \myvec{X}^{[d]} \circ \left ( {\myvec{K}^{[d+1]}}^\intercal \dfrac{\myvec{Y}}{\myvec{K}^{[d+1]} \myvec{X}^{[d]}} \right ), \\ 
\tilde{\myvec{\Gamma}}^{[d]} &\coloneqq {\myvec{K}^{[d+1]}}^\intercal \myvec{1}_{n\times m} + \lambda \myvec{1}_{p\times m},
\end{align*}
which - in connection with the update rule for $\myvec{X}$ of the previous theorem - leads to
\begin{align}
\myvec{X}^{[d+1]} &= \sqrt{\dfrac{1}{\nu}\tilde{\myvec{\Lambda}}^{[d]} + \dfrac{1}{4\nu^2} \tilde{\myvec{\Gamma}}^{[d]} \circ \tilde{\myvec{\Gamma}}^{[d]}} - \dfrac{1}{2\nu} \tilde{\myvec{\Gamma}}^{[d]} \nonumber \\
%		&=\dfrac{\left (\sqrt{\dfrac{1}{\nu}\tilde{\myvec{\Lambda}}^{[d]} \!+\! \dfrac{1}{4\nu^2} \tilde{\myvec{\Gamma}}^{[d]} \circ \tilde{\myvec{\Gamma}}^{[d]}} \!-\! \dfrac{1}{2\nu} \tilde{\myvec{\Gamma}}^{[d]} \right ) \!\circ\! \left (\sqrt{\dfrac{1}{\nu}\tilde{\myvec{\Lambda}}^{[d]} \!+\! \dfrac{1}{4\nu^2} \tilde{\myvec{\Gamma}}^{[d]} \circ \tilde{\myvec{\Gamma}}^{[d]}} \!+\! \dfrac{1}{2\nu} \tilde{\myvec{\Gamma}}^{[d]} \right )}{\sqrt{\dfrac{1}{\nu}\tilde{\myvec{\Lambda}}^{[d]} \!+\! \dfrac{1}{4\nu^2} \tilde{\myvec{\Gamma}}^{[d]} \circ \tilde{\myvec{\Gamma}}^{[d]}} + \dfrac{1}{2\nu} \tilde{\myvec{\Gamma}}^{[d]}} \nonumber \\
&=\dfrac{2\tilde{\myvec{\Lambda}}^{[d]}}{\tilde{\myvec{\Gamma}}^{[d]} + \sqrt{4\nu \tilde{\myvec{\Lambda}}^{[d]} + \tilde{\myvec{\Gamma}}^{[d]} \circ \tilde{\myvec{\Gamma}}^{[d]}}}, \label{eq:satz:Alternierender Algorithmus für das NMF-Problem:UpdateX:Umgeformt}
\end{align}
which is the update rule described in \cite{demol}.\newline
By the same approach and with the surrogate functionals derived in Section \ref{sec:surrogates_for_NMF_functionals}, we obtain the update rules for the Frobenius discrepancy term, i.e. we consider the functional
\begin{align*}
F(\myvec{K}, \myvec{X}, \myvec{V}, \myvec{W}, \myvec{\beta}) &\coloneqq \dfrac{1}{2} \Vert \myvec{Y}-\myvec{KX}\Vert_F^2 + \lambda {\Vert \myvec{X} \Vert}_1 + \frac{\mu}{2} {\Vert \myvec{K} \Vert}_F^2 + \frac{\nu}{2} {\Vert \myvec{X} \Vert}_F^2 + \omega {\Vert \myvec{K} \Vert}_1 \\
&+ \dfrac{\tau}{2} \TV(\myvec{K}) + \dfrac{\sigma_{K, 1}}{2} \Vert \myvec{I} - \myvec{V}^\intercal \myvec{K} \Vert_F^2 + \dfrac{\sigma_{K, 2}}{2} \Vert \myvec{V}-\myvec{K}\Vert_F^2 \\
&+ \dfrac{\sigma_{X, 1}}{2} \Vert \myvec{I}- \myvec{X} \myvec{W}^\intercal \Vert_F^2 + \dfrac{\sigma_{X, 2}}{2} \Vert \myvec{W}-\myvec{X}\Vert_F^2 + \dfrac{\rho}{2} \Vert \myvec{u} - \myvec{YX}^\intercal \myvec{\beta} \Vert^2
\end{align*}
A monotone decrease of this functional is obtained by the following iteration in combination with the update rules for $\myvec{V},\myvec{W},\myvec{\beta}$ as in Theorem \ref{satz:Alternierender Algorithmus für das NMF-Problem} (see also Appendix \ref{app:subsec:Frobenius-norm} for more details on the derivation of these algorithms.)
\begingroup
\allowdisplaybreaks
\begin{align*}
\text{\begin{scriptsize} $ \myvec{K}^{[d+1]} $ \end{scriptsize}} &\text{\begin{scriptsize} $\!\!=\! \dfrac{{\myvec{YX}^{[d]}}^\intercal + \tau \myvec{\Psi} \circ P(\myvec{K}^{[d]})\circ Z(\myvec{K}^{[d]}) + (\sigma_{\myvec{K}, 1} + \sigma_{\myvec{K}, 2})\myvec{V}^{[d+1]}}{\tau \myvec{\Psi} \circ P(\myvec{K}^{[d]}) + \sigma_{\myvec{K}, 2}\myvec{1}_{n\times p} + \dfrac{\myvec{K}^{[d]}\myvec{X}^{[d]}{\myvec{X}^{[d]}}^\intercal \!+\! \mu \myvec{K}^{[d]} + \omega \myvec{1}_{n\times p} + \sigma_{\myvec{K}, 1}\myvec{V}^{[d+1]}{\myvec{V}^{[d+1]}}^\intercal\myvec{K}^{[d]} }{\myvec{K}^{[d]}} } $ \end{scriptsize}} \\% \label{eq:Variationen des NMF-Modells:UpdateK}\\
\text{\begin{scriptsize} $ \myvec{X}^{[d+1]} $ \end{scriptsize}} &\text{\begin{scriptsize} $\!\!=\!\dfrac{{\myvec{K}^{[d+1]}}^\intercal \myvec{Y} + (\sigma_{\myvec{X}, 1} + \sigma_{\myvec{X}, 2})\myvec{W}^{[d+1]} + \rho \myvec{\beta}^{[d]}\myvec{u}^\intercal \myvec{Y} }{\sigma_{\myvec{X}, 2} \myvec{1}_{p\times m} \!+\! \dfrac{{\myvec{K}^{[d+1]}}^\intercal \myvec{K}^{[d+1]}\myvec{X}^{[d]} \!+\! \nu\myvec{X}^{[d]} \!+\! \lambda \myvec{1}_{p\times m} \!+\! \rho \myvec{\beta}^{[d]} {\myvec{\beta}^{[d]}}^\intercal \myvec{X}^{[d]} \myvec{Y}^\intercal \myvec{Y} \!+\! \sigma_{\myvec{X}, 1} \myvec{X}^{[d]} {\myvec{W}^{[d+1]}}^\intercal \myvec{W}^{[d+1]}   }{\myvec{X}^{[d]}} }.    $ \end{scriptsize}}% \label{eq:Variationen des NMF-Modells:UpdateX}
\end{align*}
\endgroup
\section{MALDI Imaging} \label{sec:MALDI_Imaging}
As a test case we analyse MALDI imaging data (matrix assisted laser desorption/ionization) of a rat brain. MALDI imaging is a comparatively novel modality, which unravels the molecular landscape of tissue slices and allows a subsequent proteomic or metabolic analysis \cite{AG:01,CTG:97,Ketal:14}. Clustering this data reveals for example different metabolic regions of the tissue, which can be used for supporting pathological diagnosis of tumors.\newline
The data used in this paper was obtained by a MALDI imaging experiment, see Figure \ref{fig:MALDI-Imaging} for a schematic experimental setup.\newline
In our numerical experiments, we used a classical rat brain dataset which has been used in several data processing papers before \cite{AB:13,Aetal:10,Ketal:14}. It constitutes a standard test set for hyperspectral data analysis.\newline
The tissue slice was scanned at 20185 positions. At each position a full mass spectrum with 2974 m/z (mass over charge) values was collected. I.e. instead of three color channels, as it is usual in image processing, this data has 2974 channels, each  channel containing the spatial distribution of molecules having the same m/z value.\newline\begin{figure}
	\centering
	\includegraphics[width=\textwidth]{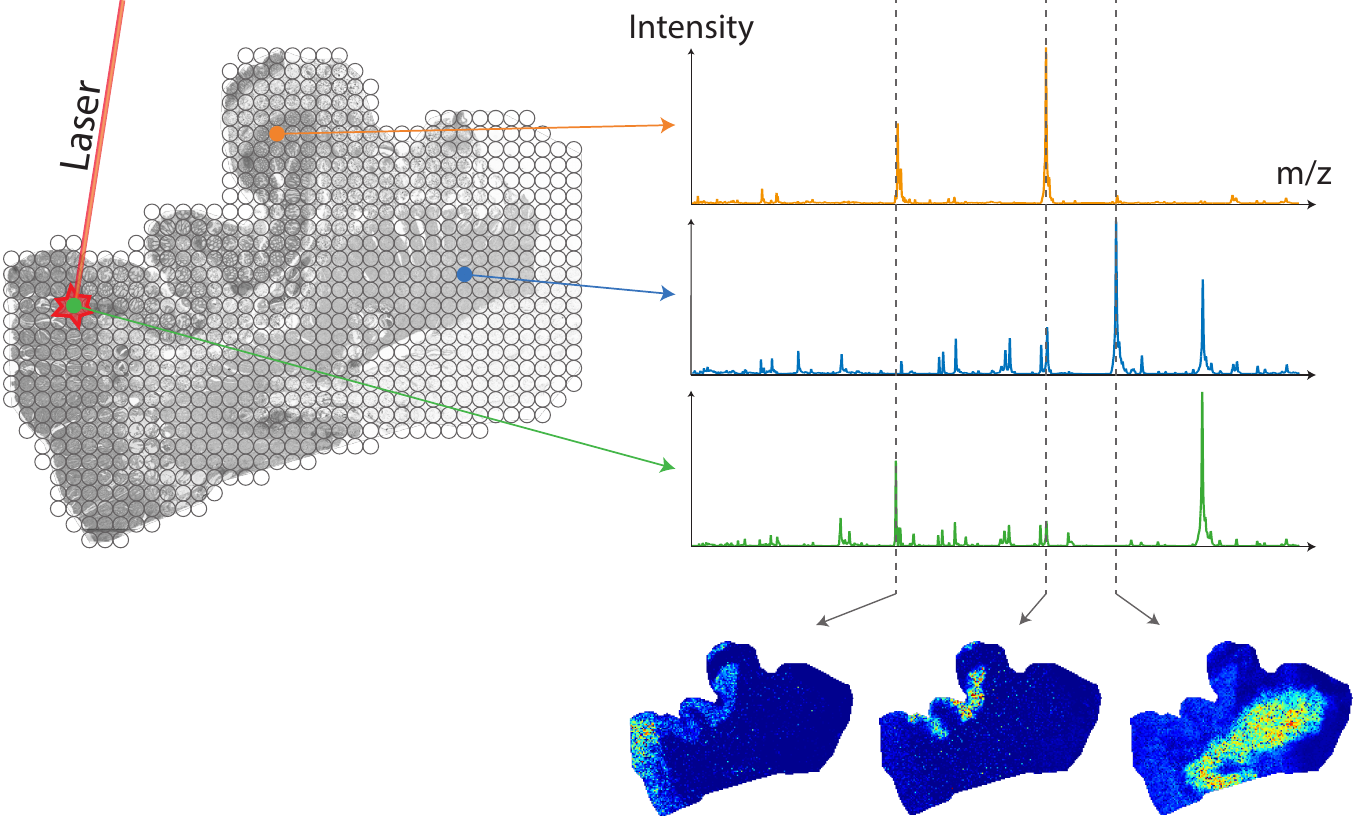}
	\caption{Structure of MALDI imaging data: A mass spectrum is obtained at different positions of a tissue slice. The full data set is a data cube, which can be visualized with different perspectives. Fixing a position of the tissue slice gives the mass spectrum at this position. Fixing a particular molecular weight reveals the distribution of molecules across the tissue slice with this weight.}
	\label{fig:MALDI-Imaging}
\end{figure}The following numerical examples were obtained with the multiplicative algorithms described in the previous section. We just illustrate the effect of the different penalty terms for some selected functionals. One can either display the columns of $\myvec{K}$ as the pseudo channels of the NMF decomposition or the rows of $\myvec{X}$ as pseudo spectra characterizing the different metabolic processes present in the tissue slice, see the Figures \ref{fig:numeric1}-\ref{fig:numeric4}.\newline
Both ways of visualization do have their respective value. Looking at the pseudo spectra in connection with orthogonality constraints leads to a clustering of the spectra and to a subdivision of the tissue slice in regions with potentially different metabolic activities, see \cite{Ketal:14}. Considering instead the different pseudo spectra, which were constructed in order to have a bases which allows a low dimensional approximaton of the data set, is the basis for subsequent proteomic analysis. E.g. one may target pseudospectra where the related pseudo channels are concentrated in regions, which were annotated by pathological experts. Mass values which are dominant in those spectra may stem from proteins/peptides relating to biomarkers as indicators for certain diseases. Hence, classification schemes based on NMF decompositions have been widely investigated \cite{leuschner18,Phon-Amnuaisuk13,TCP11}.

\begin{figure}
	\centering
	\includegraphics[width=0.8\textwidth]{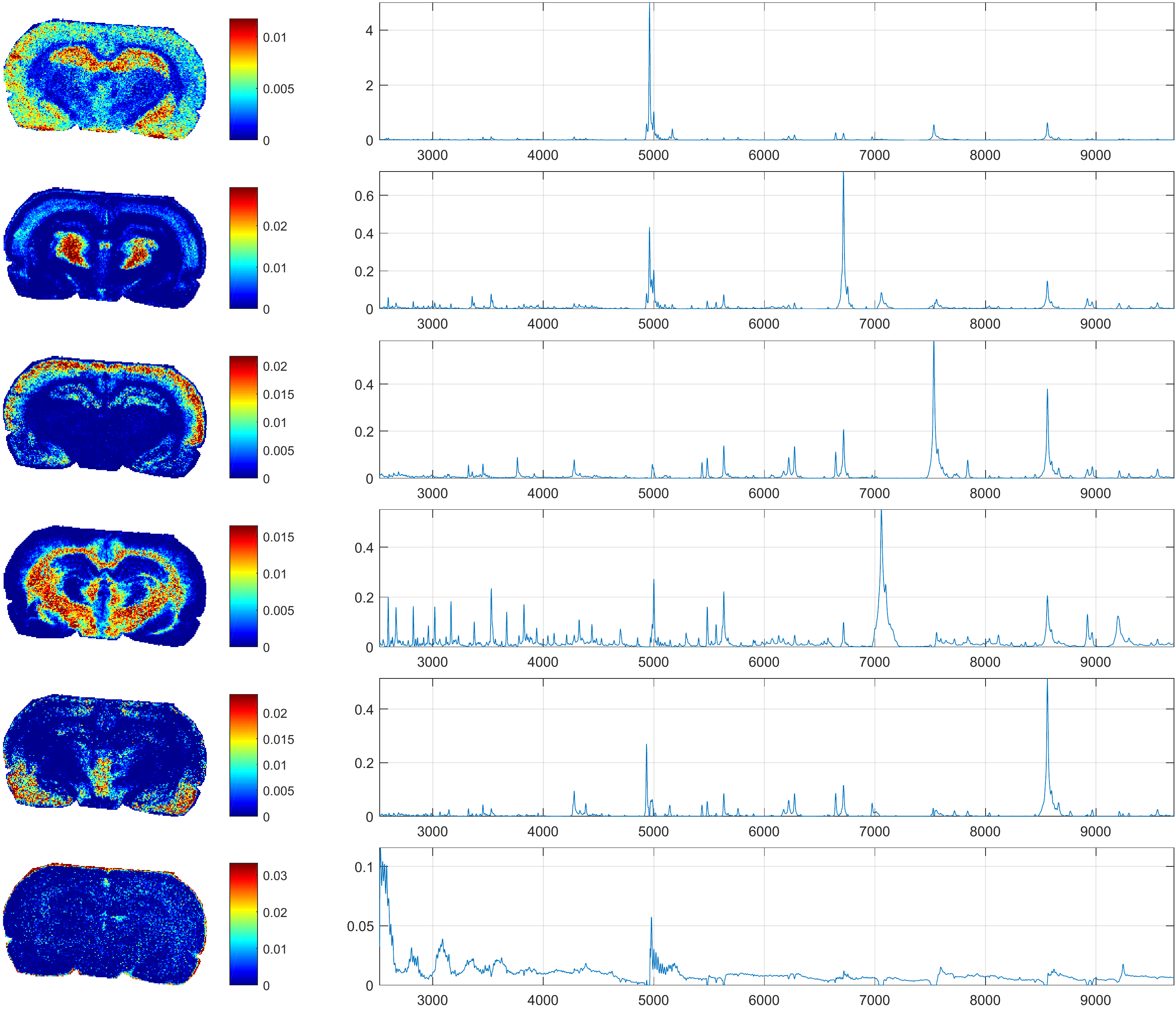}
	\caption{NMF of the rat brain dataset for $ p=6. $ Orthogonality constraints on the channels with $\sigma_{\myvec{K}, 1}=1$ and $\sigma_{\myvec{K}, 2}=1.$}
	\label{fig:numeric1}
\end{figure}

\begin{figure}
	\centering
	\includegraphics[width=0.8\textwidth]{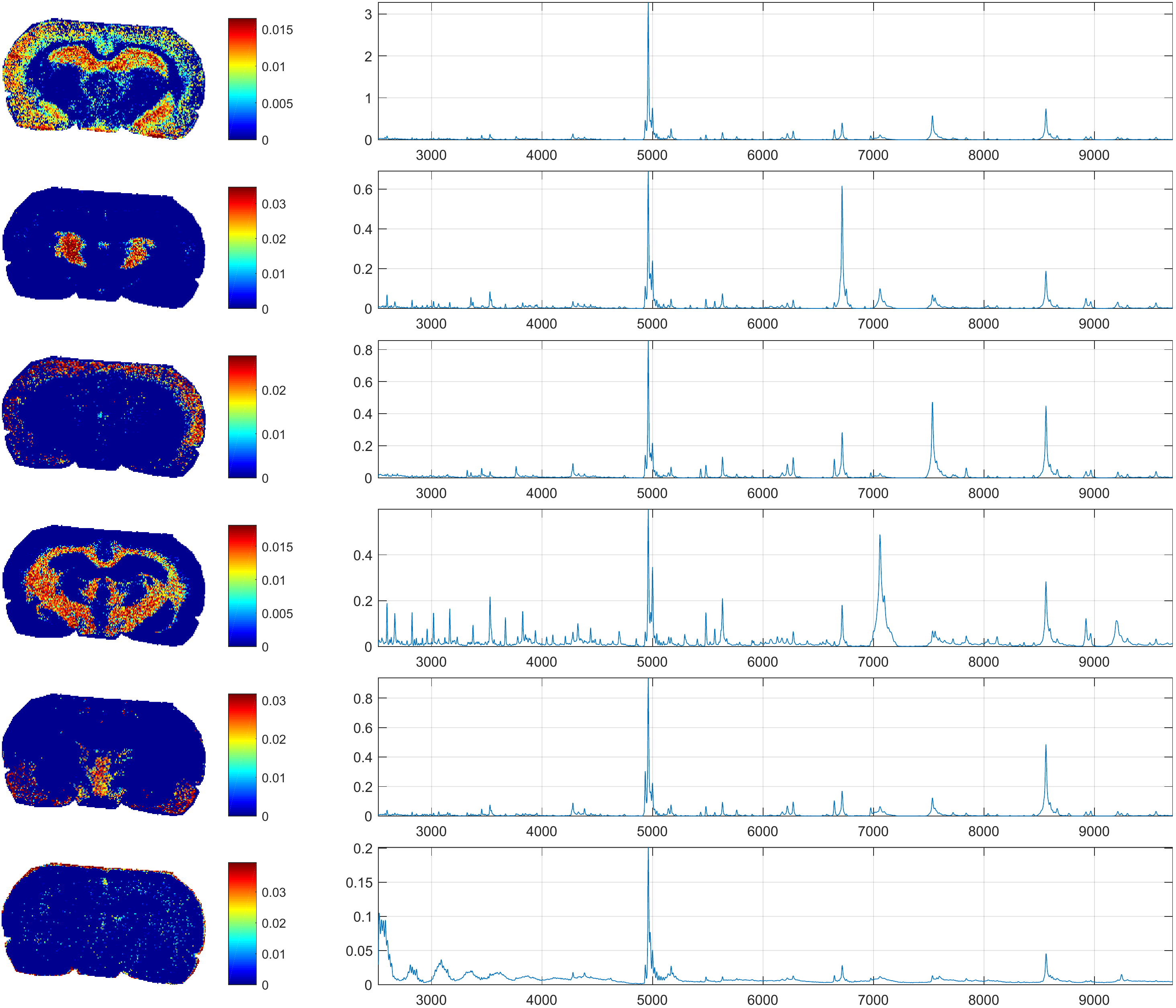}
	\caption{NMF of the rat brain dataset for $ p=6. $ Orthogonality constraints on the channels with $\sigma_{\myvec{K}, 1}=200$ and $\sigma_{\myvec{K}, 2}=200.$}
	\label{fig:numeric2}
\end{figure}

\begin{figure}
	\centering
	\includegraphics[width=0.8\textwidth]{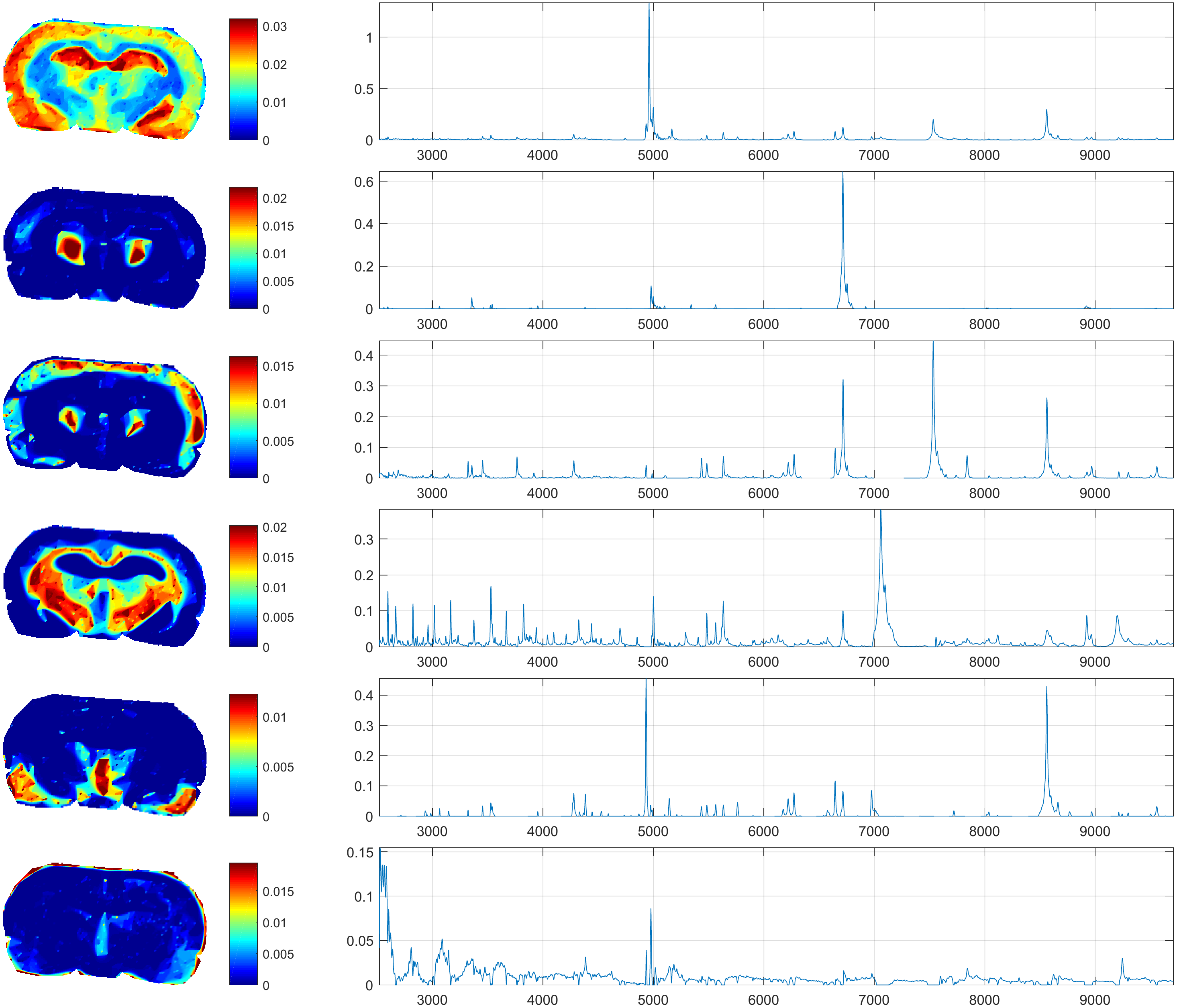}
	\caption{NMF of the rat brain dataset for $ p=6. $ Orthogonality constraints on the channels with $\sigma_{\myvec{K}, 1}=1$ and $\sigma_{\myvec{K}, 2}=1$ and $ \TV $-penalty term with $ \tau=0.4 $ and $ \varepsilon_{\TV}= 10^{-7}. $}
	\label{fig:numeric3}
\end{figure}

\begin{figure}
	\centering
	\includegraphics[width=0.8\textwidth]{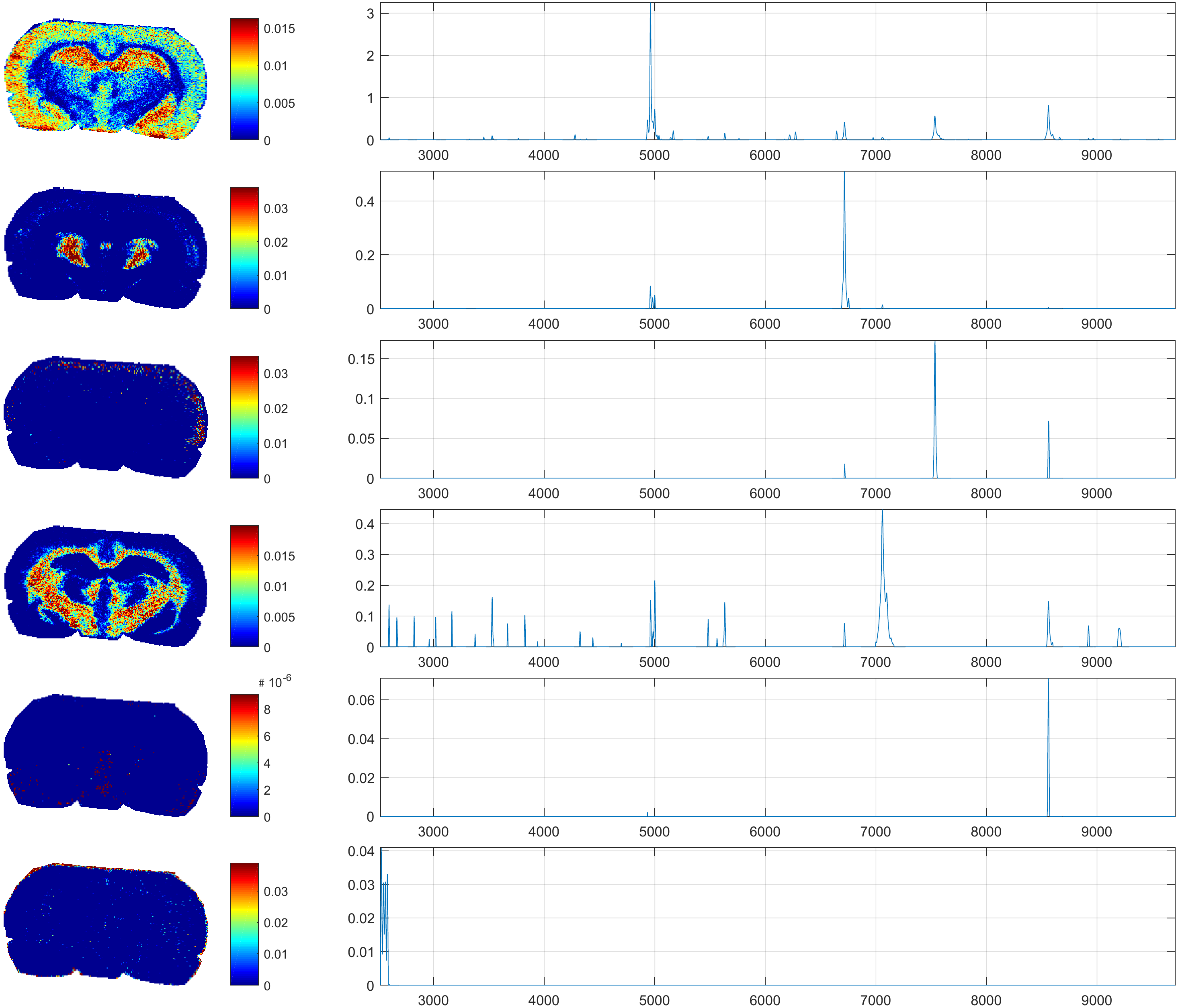}
	\caption{NMF of the rat brain dataset for $ p=6. $ Orthogonality constraints on the channels with $\sigma_{\myvec{K}, 1}=10$ and $\sigma_{\myvec{K}, 2}=10$ and sparsity penalty term on $ \myvec{X} $ with $ \lambda = 0.06. $ The sparsity penalty term has in connection with the orthogonality constraint a comparatively strong influence on the NMF computation: The sparsity in the spectra increases significantly and thus their biological interpretability, whereas the anatomic structure in the pseudo channels diminishes.}
	\label{fig:numeric4}
\end{figure}
\section{Conclusion}
In this paper, we investigated methods based on surrogate minimization approaches for the solution of NMF problems. The interest in NMF methods is related to its importance for several machine learning tasks. Application for large data sets require that the resulting algorithms are very efficient and that iteration schemes only need simple matrix-vector multiplications.\newline
The state of the art for constructing appropriate surrogates are based on case-by-case studies for the different, considered NMF models. In this paper, we embedded the different approaches in a general framework, which allowed us to analyze several extensions to the NMF cost functional, including $ \ell_1 $- and $ \ell_2 $-regularization, orthogonality constraints, total variation penalties as well as extensions, which leaded to supervised NMF concepts.\newline
Secondly, we analyzed surrogates in the context of the related  iteration schemes, which are based on first order optimality conditions. The requirement of separability as well as the need of having multiplicative updates, which preserve non-negativity without additional projections, were analyzed. This resulted in a general description of algorithms for alternating minimization of constrained NMF functionals. The potential of these methods is confirmed by numerical tests using hyperspectral data from a MALDI imaging experiment.\newline
Several further directions of research would be of interest. First of all, besides the most widely used penalty terms discussed in this paper further penalty terms, e.g. higher order TV-terms, could be considered. Secondly, construction principles for more general discrepancy terms could be analyzed (see also \cite{fevotte}).\newline
Potentially more importantly, this paper contains only very first results for combining NMF constructions directly with subsequent classification tasks. The question of an appropriate surrogate functional for the supervised NMF model with logistic regression used in \cite{leuschner18} remains unanswered and also the comparison with algorithmic alternatives such as ADMM methods needs to be explored.
%Besides linear regression (LDA), as discussed in this paper, one could use logistic regression. However, several central questions remain unanswered, e.g. the construction of a surrogate for the log-Term in the logistic regression functional is open and also the comparison with algorithmic alternatives such as ADMM methods need to be explored.

\begin{acknowledgements}
The authors want to thank Christine DeMol: her excellent presentations at several conferences and our joint discussions were the starting point for this paper. \newline
The authors would also like to thank the German Science Foundation. The work on this paper was funded by project no. GRK 2224/1 (Reasearch traning group $\pi^3$ parameter identification: analysis, algorithms, applications).\newline
The results in this paper are based on the Master thesis of the first author.
\end{acknowledgements}

\bibliographystyle{spmpsci}      % mathematics and physical sciences
\bibliography{ms}   % name your BibTeX data base

\begin{thebibliography}{10}
\providecommand{\url}[1]{{#1}}
\providecommand{\urlprefix}{URL }
\expandafter\ifx\csname urlstyle\endcsname\relax
  \providecommand{\doi}[1]{DOI~\discretionary{}{}{}#1}\else
  \providecommand{\doi}{DOI~\discretionary{}{}{}\begingroup
  \urlstyle{rm}\Url}\fi

\bibitem{AG:01}
Aebersold, R., Goodlett, D.: Mass spectrometry in proteomics.
\newblock Chem. Rev. \textbf{101}(2), 269--295 (2001)

\bibitem{AB:13}
Alexandrov, T., Bartels, A.: Testing for presence of known and unknown
  molecules in imaging mass spectrometry.
\newblock Bioinformatics \textbf{29}(18), 2335--2342 (2013)

\bibitem{Aetal:10}
Alexandrov, T., Becker, M., Deininger, S., Ernst, G., Wehder, L., Grasmair, M.,
  von Eggeling, F., Thiele, H., Maass, P.: Spatial segmentation of imaging mass
  spectrometry data with edge-preserving image denoising and clustering.
\newblock Journal of Proteome Research \textbf{9}(12), 6535--6546 (2010)

\bibitem{bishop06}
Bishop, C.: Pattern Recognition and Machine Learning.
\newblock Springer-Verlag New York (2006)

\bibitem{lindsay}
B{\"o}hning, D., Lindsay, B.G.: Monotonicity of quadratic approximation
  algorithms.
\newblock Annals of the Institute of Statistical Mathematics \textbf{40}(4),
  641--663 (1988)

\bibitem{CTG:97}
Caprioli, R., Terry, B., Gile, J.: Molecular imaging of biological samples:
  localization of peptides and proteins using maldi-tof ms.
\newblock Anal. Chem. \textbf{69}(23), 4751--4760 (1997)

\bibitem{total2}
Chambolle, A., Caselles, V., Novaga, M., Cremers, D., Pock, T.: An introduction
  to total variation for image analysis.
\newblock Theoretical Foundations and Numerical Methods for Sparse Recovery
  \textbf{9}, 263--340 (2010)

\bibitem{total1}
Chan, T., Esedoglu, S., Park, F., Yip, A.: Recent developments in total
  variation image restoration.
\newblock In: In Mathematical Models of Computer Vision. Springer Verlag (2005)

\bibitem{totaldiskret}
Condat, L.: Discrete total variation: New definition and minimization.
\newblock Research report, GIPSA-lab (2016)

\bibitem{Cvetkovski12}
Cvetkovski, Z.: Inequalities - Theorems, Techniques and Selected Problems.
\newblock Springer-Verlag Berlin Heidelberg (2012)

\bibitem{demol}
De~Mol, C.: Regularized multiplicative algorithms for nonnegative matrix
  factorization.
\newblock Methodological Aspects of Hyperspectral Imaging Workshop (2013)

\bibitem{depierro}
De~Pierro, A.R.: On the relation between the isra and the em algorithm for
  positron emission tomography.
\newblock IEEE Transactions on Medical Imaging \textbf{12}(2), 328--333 (1993)

\bibitem{defrise}
Defrise, M., Vanhove, C., Liu, X.: An algorithm for total variation
  regularization in high-dimensional linear problems.
\newblock Inverse Problems \textbf{27}(6) (2011)

\bibitem{fessler00}
Fessler, J.A.: Statistical image reconstruction methods for transmission
  tomography.
\newblock In: M.~Sonka, J.~Fitzpatrick (eds.) Handbook of Medical Imaging,
  Volume 2. Medical Image Processing and Analysis, pp. 1--70. SPIE (2000)

\bibitem{FBD09}
F{\'e}votte, C., Bertin, N., Durrieu, J.L.: Nonnegative matrix factorization
  with the itakura-saito-divergence: With application to music analysis.
\newblock Neural Computation \textbf{21}(3), 793--830 (2009)

\bibitem{fevotte}
F{\'e}votte, C., Idier, J.: Algorithms for nonnegative matrix factorization
  with the beta-divergence.
\newblock Neural Computation \textbf{23}(9), 2421--2456 (2011)

\bibitem{hennequin}
Hennequin, R., David, B., Badeau, R.: beta-divergence as a subclass of bregman
  divergence.
\newblock IEEE Signal Processing Letters \textbf{18}(2), 83--86 (2011)

\bibitem{hunter}
Hunter, D.R., Lange, K.: A tutorial on mm algorithms.
\newblock The American Statistician \textbf{58}(1), 30--37 (2004)

\bibitem{IJT11}
Ito, K., Jin, B., Takeuchi, T.: Multi-parameter tikhonov regularization -- an
  augmented approach.
\newblock Chinese Annals of Mathematics, Series B \textbf{35}(3), 383--398
  (2014)

\bibitem{Jin09}
Jin, B., Lorenz, D.A., Schiffler, S.: Elastic-net regularization: error
  estimates and active set methods.
\newblock Inverse Problems \textbf{25}(11) (2009)

\bibitem{bangti12}
Jin, B., Maass, P.: Sparsity regularization for parameter identification
  problems.
\newblock Inverse Problems \textbf{28}(12) (2012)

\bibitem{Ketal:14}
Kobarg, J.H., Maass, P., Oetjen, J., Hirsch, E., Sagiv, C., Golbabaee, M.,
  Vandergheynst, P.: Numerical experiments with maldi imaging data.
\newblock Advances in Computational Mathematics \textbf{40}(3), 667--682 (2014)

\bibitem{lange}
Lange, K.: Optimization, \emph{Springer Texts in Statistics}, vol.~95, 2 edn.
\newblock Springer-Verlag New York (2013)

\bibitem{lecharlier13}
Lecharlier, L., Mol, C.D.: Regularized blind deconvolution with poisson data.
\newblock Journal of Physics: Conference Series \textbf{464}(1) (2013)

\bibitem{leeseung}
Lee, D.D., Seung, H.S.: Learning the parts of objects by non-negative matrix
  factorization.
\newblock Nature \textbf{401}, 788--791 (1999)

\bibitem{leuschner18}
Leuschner, J., Fernsel, P., Schmidt, M., Lachmund, D., Boskamp, T., Maass, P.:
  Supervised non-negative matrix factorization methods with maldi-imaging
  applications.
\newblock Bioinformatics (in review)  (2018)

\bibitem{LD13}
Li, T., Ding, C.: Non-negative matrix factorization for clustering: A survey.
\newblock In: Data Clustering: Algorithms and Applications, pp. 149--176 (2013)

\bibitem{louis89}
Louis, A.K.: Inverse und schlecht gestellte Probleme.
\newblock Vieweg+Teubner Verlag (1989)

\bibitem{oliveira}
Oliveira, J.P., Bioucas-Dias, J.M., Figueiredo, M.A.T.: Review: Adaptive total
  variation image deblurring: A majorization-minimization approach.
\newblock Signal Process. \textbf{89}(9), 1683--1693 (2009)

\bibitem{Phon-Amnuaisuk13}
Phon-Amnuaisuk, S.: Applying non-negative matrix factorization to classify
  superimposed handwritten digits.
\newblock Procedia Computer Science \textbf{24}, 261--267 (2013)

\bibitem{rudin}
Rudin, L.I., Osher, S., Fatemi, E.: Nonlinear total variation based noise
  removal algorithms.
\newblock Physica D \textbf{60}(1-4), 259--268 (1992)

\bibitem{sun}
Sun, D.L., F{\'e}votte, C.: Alternating direction method of multipliers for
  non-negative matrix factorization with the beta-divergence.
\newblock In: 2014 IEEE International Conference on Acoustics, Speech and
  Signal Processing (ICASSP), pp. 6201--6205 (2014)

\bibitem{TF12}
Tan, V.Y.F., F{\'e}votte, C.: Automatic relevance determination in nonnegative
  matrix factorization with the beta-divergence.
\newblock arXiv preprint \textbf{arXiv: 1111.6085v3} (2012)

\bibitem{TCP11}
Tang, J., Ceng, X., Peng, B.: New methods of data clustering and classification
  based on nmf.
\newblock In: 2011 International Conference on Business Computing and Global
  Informatization, pp. 432--435 (2011)

\bibitem{ZKY05}
Zhang, Z., Kwok, J.T., Yeung, D.Y.: Surrogate maximization/minimization
  algorithms for adaboost and the logistic regression model.
\newblock In: Proceedings of the Twenty-first International Conference on
  Machine Learning, ICML '04 (2004)

\bibitem{zhang}
Zhang, Z., Kwok, J.T., Yeung, D.Y.: Surrogate maximization/minimization
  algorithms and extensions.
\newblock Machine Learning \textbf{69}(1), 1--33 (2007)

\end{thebibliography}
%\begin{thebibliography}{}
%%
%% and use \bibitem to create references. Consult the Instructions
%% for authors for reference list style.
%\end{thebibliography}

%Start of the Appendix
\newpage
\begin{appendices}
\section{Details on the Derivation of the Algorithms in Section \ref{sec:Surrogate_based_NMF_algorithms}} \label{app:sec:derivation_of_the_update_rules}
In this section, we give a more detailed derivation of the algorithms presented in Section \ref{sec:Surrogate_based_NMF_algorithms}. We start with the less complex case of the Frobenius norm as discrepancy term and then turn to the Kullback-Leibler divergence. To cover both aspects, we derive the update rules of $ \myvec{X} $ for the Frobenius discrepancy term and of $ \myvec{K} $ in the case of the KLD. We will also take a closer look at the effect of $ \kappa $ in equation \eqref{eq:Low Quadratic Bound Principle:Lambda} with respect to the LQBP construction principle.

\subsection{Frobenius Norm} \label{app:subsec:Frobenius-norm}
We consider the general cost function described in Section \ref{sec:Surrogate_based_NMF_algorithms} for the case of the Frobenius norm. To compute the update rules for $ \myvec{X},$ it is enough to examine the function
\begin{align*}
F(\myvec{X}) &\coloneqq \underbrace{\dfrac{1}{2}\Vert \myvec{Y}-\myvec{KX}\Vert_F^2 + \lambda \Vert \myvec{X} \Vert_1}_{\eqqcolon F_1(\myvec{X})} + \frac{\nu}{2} {\Vert \myvec{X} \Vert}_F^2 \\
%&+ \dfrac{\tau}{2} \TV(\myvec{K}) + \dfrac{\sigma_{K, 1}}{2} \Vert \myvec{I} - \myvec{V}^\intercal \myvec{K} \Vert_F^2 + \dfrac{\sigma_{K, 2}}{2} \Vert \myvec{V}-\myvec{K}\Vert_F^2 \\
&+ \dfrac{\sigma_{\myvec{X}, 1}}{2} \Vert \myvec{I}- \myvec{X} \myvec{W}^\intercal \Vert_F^2 + \dfrac{\sigma_{\myvec{X}, 2}}{2} \Vert \myvec{W}-\myvec{X}\Vert_F^2 + \dfrac{\rho}{2} \Vert \myvec{u} - \myvec{YX}^\intercal \myvec{\beta} \Vert^2,
\end{align*}
where all terms independent from $ \myvec{X} $ are omitted. Based on Remark \ref{rem:Addition von Surrogat-Funktionen} and following the discussion of Section \ref{subsec:Surrogates for lp-penalty terms}, the construction of a surrogate for $ F $ can be done separately for $ F_1 $ and the remaining penalty terms.\newline
The construction of a surrogate for $ F_1 $ with the LQBP principle as it has been done similarly in Subsection \ref{subsec:Surrogates for lp-penalty terms} is essential. If we would use instead a surrogate for the discrepancy term $ \nicefrac{1}{2} \Vert \myvec{Y}-\myvec{KX}\Vert_F^2 $ from Subsection \ref{subsec:Frobenius discrepancy and LQBP} or \ref{subsec: Frobenius_discrepancy_and_Jensens_inequality} and take the $ \ell_1 $-penalty term $\lambda \Vert \myvec{X} \Vert_1 $ as surrogate itself, it is easy to see that this would not lead to multiplicative update rules. It is the $ \ell_1 $-penalty term which causes the difficulty. Computing the first order optimality condition for the corresponding surrogate $ \tilde{Q}_F(\myvec{X}, \myvec{A}) \eqqcolon \lambda \Vert \myvec{X} \Vert_1 + \hat{Q}_F(\myvec{X}, \myvec{A}) $ with respect to $ \myvec{X} $ would lead to
\begin{equation*}
	0 = \dfrac{\partial \tilde{Q}_F}{\partial X_{\xi \zeta}}(\myvec{X}, \myvec{A}) = \lambda + \dfrac{\partial \hat{Q}_F}{\partial X_{\xi \zeta}}(\myvec{X}, \myvec{A}),
\end{equation*}
where the second term on the right hand side does not depend on $ \lambda. $ Hence, we get a sign in front of $ \lambda $ by solving the equation for $ X_{\xi \zeta} $ and we will not obtain multiplicative updates for $ \myvec{X}. $\newline
A correct surrogate is obtained by using the LQBP principle to $ F_1 $ and leads to
\begin{equation*}
	Q_F(\myvec{X}, \myvec{A}) \coloneqq Q_{F_1}(\myvec{X}, \myvec{A}) + \dfrac{\nu}{2} \Vert \myvec{X} \Vert_F^2 + Q_{\text{Orth}}(\myvec{X}, \myvec{A}) + Q_{\text{LR}}(\myvec{X}, \myvec{A})
\end{equation*}
with
\begin{align*}
	Q_{F_1}(\myvec{X}, \myvec{A}) = \sum_{j=1}^m & \ f_{\myvec{Y}_{\bullet,j}}(\myvec{A}_{\bullet, j}) + \nabla f_{\myvec{Y}_{\bullet,j}}(\myvec{A}_{\bullet, j})^\intercal (\myvec{X}_{\bullet, j} - \myvec{A}_{\bullet, j})\\
	&+ \dfrac{1}{2} (\myvec{X}_{\bullet, j} - \myvec{A}_{\bullet, j})^\intercal \myvec{\Lambda}_{f_{\myvec{Y}_{\bullet,j}}}(\myvec{A}_{\bullet, j})(\myvec{X}_{\bullet, j} - \myvec{A}_{\bullet, j}),
\end{align*}
where $ f_{\myvec{Y}_{\bullet,j}}:\mathbb{R}_{\geq 0}^p \to \mathbb{R} $ is defined as
\begin{equation*}
	f_{\myvec{Y}_{\bullet,j}}(\myvec{x}) \coloneqq \dfrac{1}{2} \Vert \myvec{Y}_{\bullet,j}-\myvec{Kx}\Vert^2 + \lambda \Vert \myvec{x} \Vert_1
\end{equation*}
and with the diagonal matrix
\begin{equation*}
	\Lambda_{f_{\myvec{Y}_{\bullet,j}}}(\myvec{A}_{\bullet, j})_{kk} = \dfrac{(\nabla^2 f_{{\myvec{Y}_{\bullet,j}}}(\myvec{A}_{\bullet, j})\ A_{\bullet, j})_k + \kappa_k}{A_{k j}} = \dfrac{(K^\intercal K  A_{\bullet, j})_k + \kappa_k}{A_{k j}}.
\end{equation*}
The functionals $ Q_{\text{Orth}}$ resp. $ Q_{\text{LR}}$ are the surrogates obtained from Theorem \ref{satz:Surrogate functionals for orthogonal NMF} resp. Theorem \ref{satz:Surrogat-Funktionen zur linearen Regression (SPS)}. It will turn out that an appropriate choice of $ \kappa_k $ will ensure a multiplicative NMF algorithm.\newline
The computation of the first order optimality condition for $ Q_F $ leads to
\begin{align*}
	0 = \dfrac{\partial Q_F}{\partial X_{\xi \zeta}}(\myvec{X}, \myvec{A}) &= (K^\intercal KA)_{\xi \zeta} - (K^\intercal Y)_{\xi \zeta} + \lambda + \dfrac{(K^\intercal KA)_{\xi \zeta} + \kappa_\xi}{A_{\xi, \zeta}}(X_{\xi \zeta} - A_{\xi \zeta})\\
	&+\sigma_{\myvec{X}, 1} \sum_{k=1}^p W_{k\zeta}\left ( \dfrac{X_{\xi \zeta}}{A_{\xi \zeta}} (AW^\intercal)_{\xi k} - \delta_{\xi k} \right ) + \sigma_{\myvec{X}, 2} (X_{\xi \zeta} - W_{\xi \zeta})\\
	&+\rho \beta_\xi \sum_{i=1}^n Y_{i \zeta} \left ( \dfrac{X_{\xi \zeta}}{A_{\xi \zeta}} (YA^\intercal\beta)_i - u_i \right ) + \nu X_{\xi \zeta}.
\end{align*}
One can see immediately, that the choice of $ \kappa_\xi \coloneqq \lambda $ for all $ \xi \in \{1, \dots, p\} $ is appropriate to get rid of the problematic term $ \lambda. $ Hence, we obtain
\begin{align*}
	0 &= - (K^\intercal Y)_{\xi \zeta} + \dfrac{X_{\xi \zeta}}{A_{\xi \zeta}} \left ((K^\intercal KA)_{\xi \zeta} + \lambda \right )\\
	&+\sigma_{\myvec{X}, 1} \sum_{k=1}^p W_{k\zeta}\left ( \dfrac{X_{\xi \zeta}}{A_{\xi \zeta}} (AW^\intercal)_{\xi k} - \delta_{\xi k} \right ) + \sigma_{\myvec{X}, 2} (X_{\xi \zeta} - W_{\xi \zeta})\\
	&+\rho \beta_\xi \sum_{i=1}^n Y_{i \zeta} \left ( \dfrac{X_{\xi \zeta}}{A_{\xi \zeta}} (YA^\intercal\beta)_i - u_i \right ) + \nu X_{\xi \zeta}.
\end{align*}
Reordering the terms leads to
\begin{align*}
	&(K^\intercal Y)_{\xi \zeta} + (\sigma_{\myvec{X}, 1} + \sigma_{\myvec{X}, 2}) W_{\xi \zeta} + \rho \beta_\xi (Y^\intercal u)_\zeta \\
	&= \dfrac{X_{\xi \zeta}}{A_{\xi \zeta}} \big( (K^\intercal KA)_{\xi \zeta} + \nu A_{\xi \zeta} + \lambda + \rho \beta_\xi (Y^\intercal Y A^\intercal \beta)_\zeta + \sigma_{\myvec{X}, 1} (AW^\intercal W)_{\xi \zeta} + \sigma_{\myvec{X}, 2}A_{\xi\zeta} \big).
\end{align*}
Solving for $ X_{\xi \zeta} $ and extending the equation to the whole matrix $ \myvec{X} $ yields finally
\begin{align*}
	\myvec{X} = \myvec{A} \circ \dfrac{\myvec{K}^\intercal \myvec{Y} + (\sigma_{\myvec{X}, 1} + \sigma_{\myvec{X}, 2}) \myvec{W} + \rho \myvec{\beta u}^\intercal \myvec{Y}}{\myvec{K}^\intercal \myvec{KA} + (\sigma_{\myvec{X}, 2} + \nu) \myvec{A} + \lambda \myvec{1}_{p\times m} + \rho \myvec{\beta\beta}^\intercal \myvec{AY}^\intercal\myvec{Y} + \sigma_{\myvec{X}, 1}\myvec{AW}^\intercal\myvec{W}}.
\end{align*}
By exploiting the surrogate minimization principle as described in Lemma \ref{lem:Monotoner Abfall durch Surrogat-Funktion}, we get finally the update rule for $ \myvec{X} $ presented in Section \ref{sec:Surrogate_based_NMF_algorithms}.
\subsection{Kullback-Leibler Divergence} \label{app:subsec:KLD}
We take Equation \eqref{eq:NMF-Problem mit KLD:PartielleAbleitungQF} as our starting point. The computation of the first order optimality condition gives
\begin{align*}
	\dfrac{\partial Q_F}{\partial K_{\xi \zeta}}(\myvec{K}, \myvec{A}) &= \sum_{j=1}^m \left [X_{\zeta j} - \dfrac{Y_{\xi j}}{(AX)_{\xi j}} A_{\xi \zeta} X_{\zeta j} \dfrac{1}{K_{\xi \zeta}}\right ] + \mu K_{\xi \zeta} + \omega + \sigma_{\myvec{K}, 2} (K_{\xi \zeta} - V_{\xi \zeta})\\
	 &+ \tau \psi_\zeta P(\myvec{A})_{\xi \zeta} (K_{\xi \zeta} - Z(\myvec{A})_{\xi \zeta}) + \sigma_{\myvec{K}, 1} \sum_{k=1}^p V_{\xi k} \left ( \dfrac{K_{\xi \zeta}}{A_{\xi \zeta}} (V^\intercal A)_{k\zeta} - \delta_{k\zeta} \right )\\
	 &=0.
\end{align*}
Multiplying on both sides with $ K_{\xi \zeta} $ and sorting the terms already gives the system of quadratic equations mentioned in Section \ref{sec:Surrogate_based_NMF_algorithms}, namely
\begin{align*}
&K_{\xi \zeta}^2 \left (\mu + \tau \psi_\zeta P(\myvec{A})_{\xi \zeta} + \dfrac{\sigma_{\myvec{K}, 1}}{A_{\xi \zeta}} (VV^\intercal A)_{\xi \zeta} + \sigma_{\myvec{K}, 2}\right ) \\
+&K_{\xi \zeta} \left ( \sum_{j=1}^{m} X_{\zeta j} + \omega - \tau \psi_\zeta P(\myvec{A})_{\xi \zeta} Z(\myvec{A})_{\xi \zeta} - (\sigma_{\myvec{K}, 1} + \sigma_{\myvec{K}, 1}) V_{\xi \zeta}  \right ) \\
=&A_{\xi \zeta} \sum_{j=1}^{m} \dfrac{Y_{\xi j}}{(AX)_{\xi j}} X_{\zeta j}.
\end{align*}
Taking into account that
\begin{equation*}
	\sum_{j=1}^{m} \dfrac{Y_{\xi j}}{(AX)_{\xi j}} X_{\zeta j} = \left ( \dfrac{Y}{AX} X^\intercal \right)_{\xi \zeta} \quad \text{and} \quad \sum_{j=1}^{m} X_{\zeta j} = (1_{n\times m}X^\intercal)_{\xi \zeta},
\end{equation*}
we obtain the explicit solution of $ K_{\xi \zeta} $ by completing the square and get
\begin{align*}
	K_{\xi \zeta} &= \left [  \dfrac{A_{\xi \zeta}}{\mu + \tau \psi_\zeta P(\myvec{A})_{\xi \zeta} + \dfrac{\sigma_{\myvec{K}, 1}}{A_{\xi \zeta}} (VV^\intercal A)_{\xi \zeta} + \sigma_{\myvec{K}, 2}} \left ( \dfrac{Y}{AX} X^\intercal \right)_{\xi \zeta} \right.\\
	&\left. + \dfrac{1}{4} \left ( \dfrac{(1_{n\times m}X^\intercal)_{\xi \zeta} + \omega - \tau \psi_\zeta P(\myvec{A})_{\xi \zeta} Z(\myvec{A})_{\xi \zeta} - V_{\xi \zeta}(\sigma_{\myvec{K}, 1} + \sigma_{\myvec{K}, 1})}{\mu + \tau \psi_\zeta P(\myvec{A})_{\xi \zeta} + \dfrac{\sigma_{\myvec{K}, 1}}{A_{\xi \zeta}} (VV^\intercal A)_{\xi \zeta} + \sigma_{\myvec{K}, 2}} \right )^2 \ \right ]^{\nicefrac{1}{2}}\\
	&-\dfrac{1}{2} \left ( \dfrac{(1_{n\times m}X^\intercal)_{\xi \zeta} + \omega - \tau \psi_\zeta P(\myvec{A})_{\xi \zeta} Z(\myvec{A})_{\xi \zeta} - V_{\xi \zeta}(\sigma_{\myvec{K}, 1} + \sigma_{\myvec{K}, 1})}{\mu + \tau \psi_\zeta P(\myvec{A})_{\xi \zeta} + \dfrac{\sigma_{\myvec{K}, 1}}{A_{\xi \zeta}} (VV^\intercal A)_{\xi \zeta} + \sigma_{\myvec{K}, 2}} \right ).
\end{align*}
This equation holds for arbitrary $ \xi\in \{1, \dots, n \} $ and $ \zeta\in \{ 1,\dots,k \}. $ We therefore can extend this relation to the whole matrix $ \myvec{K} $ and obtain
\begin{align*}
	\text{\begin{normalsize} $\myvec{K}$ \end{normalsize}} &\text{\begin{normalsize} $=\left [ \left ( \dfrac{\myvec{A}}{\mu \myvec{1}_{n\times p} + \tau \myvec{\Psi} \circ P(\myvec{A}) + \sigma_{\myvec{K}, 1} \dfrac{\myvec{VV}^\intercal \myvec{A}}{\myvec{A}} + \sigma_{\myvec{K}, 2} \myvec{1}_{n\times p} } \right ) \circ \left ( \dfrac{\myvec{Y}}{\myvec{A}\myvec{X}}\myvec{X}^\intercal \right )\right .$ \end{normalsize}} \\
	&\text{\begin{normalsize} $\left . +\dfrac{1}{4} \left ( \dfrac{\myvec{1}_{n\times m}\myvec{X}^\intercal + \omega\myvec{1}_{n\times p} - \tau \myvec{\Psi} \circ P(\myvec{A}) \circ Z(\myvec{A}) - (\sigma_{\myvec{K}, 1} + \sigma_{\myvec{K}, 2}) \myvec{V} }{\mu \myvec{1}_{n\times p} + \tau \myvec{\Psi} \circ P(\myvec{A}) + \sigma_{\myvec{K}, 1} \dfrac{\myvec{VV}^\intercal \myvec{A}}{\myvec{A}} + \sigma_{\myvec{K}, 2} \myvec{1}_{n\times p}} \right )^2 \ \right ]^{\nicefrac{1}{2}}$ \end{normalsize}} \\
	&\text{\begin{normalsize} $-\dfrac{1}{2} \left ( \dfrac{\myvec{1}_{n\times m}\myvec{X}^\intercal + \omega\myvec{1}_{n\times p} - \tau \myvec{\Psi} \circ P(\myvec{A}) \circ Z(\myvec{A}) - (\sigma_{\myvec{K}, 1} + \sigma_{\myvec{K}, 2}) \myvec{V} }{\mu \myvec{1}_{n\times p} + \tau \myvec{\Psi} \circ P(\myvec{A}) + \sigma_{\myvec{K}, 1} \dfrac{\myvec{VV}^\intercal \myvec{A}}{\myvec{A}} + \sigma_{\myvec{K}, 2} \myvec{1}_{n\times p}} \right ),$ \end{normalsize}}
\end{align*}
which is exactly the described update rule in Section \ref{sec:Surrogate_based_NMF_algorithms}.
\section{Kullback-Leibler Divergence Discrepancy and LQBP} \label{app:sec:KLD_and_LQBP}
In this Section, we will use the LQBP construction principle to derive a multiplicative algorithm for the cost function
\begin{equation*}
	F(\myvec{X}) \coloneqq \KL(\myvec{Y}, \myvec{KX}) = \sum_{j=1}^m \KL(\myvec{Y}_{\bullet, j},\myvec{K}\myvec{X}_{\bullet, j}) \eqqcolon \sum_{j=1}^m f_{\myvec{Y}_{\bullet, j}}(\myvec{X}_{\bullet, j})
\end{equation*}
Similar to the approach in Appendix $ \ref{app:subsec:KLD}, $ we define according to the LQBP principle the surrogate
\begin{align*}
Q_{F}(\myvec{X}, \myvec{A}) = \sum_{j=1}^m & \ f_{\myvec{Y}_{\bullet,j}}(\myvec{A}_{\bullet, j}) + \nabla f_{\myvec{Y}_{\bullet,j}}(\myvec{A}_{\bullet, j})^\intercal (\myvec{X}_{\bullet, j} - \myvec{A}_{\bullet, j})\\
&+ \dfrac{1}{2} (\myvec{X}_{\bullet, j} - \myvec{A}_{\bullet, j})^\intercal \myvec{\Lambda}_{f_{\myvec{Y}_{\bullet,j}}}(\myvec{A}_{\bullet, j})(\myvec{X}_{\bullet, j} - \myvec{A}_{\bullet, j})
\end{align*}
with the diagonal matrix
\begin{equation*}
\Lambda_{f_{\myvec{Y}_{\bullet,j}}}(\myvec{A}_{\bullet, j})_{kk} = \dfrac{(\nabla^2 f_{{\myvec{Y}_{\bullet,j}}}(\myvec{A}_{\bullet, j})\ A_{\bullet, j})_k + \kappa_k}{A_{k j}}.
\end{equation*}
It follows for the partial derivatives of $ f $
\begin{align*}
\dfrac{\partial f_{\myvec{Y}_{\bullet, \zeta}}}{\partial X_{\beta \zeta}}(\myvec{X}_{\bullet, \zeta}) &= -\sum_{i=1}^n \dfrac{Y_{i\zeta} K_{i\beta}}{(KX)_{i\zeta}} + K_{i\beta},\\
\dfrac{\partial^2 f_{\myvec{Y}_{\bullet, \zeta}}}{\partial X_{\alpha \zeta}\partial X_{\beta \zeta}}(\myvec{X}_{\bullet, \zeta}) &=\sum_{i=1}^n \dfrac{Y_{i\zeta} K_{i\alpha} K_{i\beta}}{(KX)_{i\zeta}^2}.
\end{align*}
The first order optimality condition of the surrogate functional leads then to
\begin{align*}
0 = \dfrac{\partial Q_F}{\partial X_{\xi \zeta}}(\myvec{X}, \myvec{A}) &= -\sum_{i=1}^n \dfrac{Y_{i\zeta} K_{i\xi}}{(KA)_{i\zeta}} + \sum_{i=1}^n K_{i\xi}\\
&+ \dfrac{\sum_{i=1}^n \dfrac{Y_{i\zeta} K_{i\xi}}{(KA)_{i\zeta}} + \kappa_\xi}{A_{\xi \zeta}}(X_{\xi\zeta} - A_{\xi\zeta}).
\end{align*}
Setting $ \kappa_\xi \coloneqq \sum_{i=1}^n K_{i\xi}$ and solving for $ X_{\xi \zeta} $ leads finally to the multiplicative update rule
\begin{equation*}
	\myvec{X}^{[d+1]} = \dfrac{2 \myvec{X}^{[d]}}{{\myvec{K}^{[d+1]}}^\intercal \dfrac{\myvec{Y}}{\myvec{K}^{[d+1]}\myvec{X}^{[d]}} + {\myvec{K}^{[d+1]}}^\intercal \myvec{1}_{n\times m}} \circ {\myvec{K}^{[d+1]}}^\intercal \dfrac{\myvec{Y}}{\myvec{K}^{[d+1]}\myvec{X}^{[d]}},
\end{equation*}
which differs from the classical update rule for the KLD described in \cite{demol,lecharlier13,leeseung}.
\section{Surrogate of the TV Penalty - Separability} \label{app:sec:TV_Separability}
In this section, we will prove the separability of the surrogate functional 
\begin{equation}\label{eq:prf:Separabilität von QTV:QTV}
	\begin{aligned}
		\text{ \begin{footnotesize} $Q_{\TV}(\myvec{K}, \myvec{A}) = \sum_{k=1}^p \psi_k  \sum_{i=1}^n \dfrac{1}{\vert \nabla_{ik} A \vert} \bigg( \varepsilon_{\TV}^2 + \sum_{\ell\in N_i} \mkern-25mu $ \end{footnotesize} } & \text{ \begin{footnotesize} $K_{ik}^2 + K_{\ell k}^2 - K_{ik} ( A_{ik} + A_{\ell k} ) $ \end{footnotesize} }  \\
		&\text{ \begin{footnotesize} $ - K_{\ell k} ( A_{ik} + A_{\ell k} ) + A_{ik}^2 + A_{\ell k}^2 \bigg).$ \end{footnotesize} }
	\end{aligned}
\end{equation}
described in Theorem \ref{satz:Surrogat-Funktion zum TV-Strafterm}. Furthermore, we choose an arbitrary $ s \in \{ 1, \dots, n \}$ and $ t\in \{1, \dots, p\}. $ The aim is now to find all terms in \eqref{eq:prf:Separabilität von QTV:QTV} with $ K_{st}^2$ and $ K_{st}. $\newline
To find all quadratic terms $ K_{st}^2$ in \eqref{eq:prf:Separabilität von QTV:QTV}, we see that we have to fix the index $ k, $ such that $ k=t. $ The remaining indices in \eqref{eq:prf:Separabilität von QTV:QTV}, which have to be analyzed, are $ i $ and $ \ell. $\newline
For the case $ i=s,$ we find that the preceding coefficient is
\begin{equation*}
	\dfrac{\psi_t}{\vert \nabla_{st} \myvec{A} \vert} \sum_{r\in N_s} 1.
\end{equation*}
The case $ \ell=s $ can only occur for those indices $ i, $ which satisfy $ s\in N_i. $ The definition of the adjoint neighbourhood pixels gives
\begin{equation*} %\label{eq:prf:Separabilität von QTV:Äquivalenz}
\forall i: \ s\in N_i \quad \Leftrightarrow \quad \forall i: \ i\in \bar{N}_s.
\end{equation*}
Therefore, the corresponding preceding coefficient is here
\begin{equation*}
	\psi_t \sum_{r\in \bar{N}_s} \dfrac{1}{\vert \nabla_{rt} A \vert}.
\end{equation*}
Altogether, we obtain for the quadratic terms $ K_{st}^2$ the coefficient
\begin{equation*} % \label{eq:prf:Separabilität von QTV:PWelle}
	\tilde{P}_{s t}(\vec{A}) \coloneqq \psi_t \left ( \dfrac{1}{\vert \nabla_{st} \myvec{A} \vert} \sum_{r\in N_s} 1 + \sum_{r\in \bar{N}_s} \dfrac{1}{\vert \nabla_{rt} \myvec{A} \vert} \right ),
\end{equation*}
such that $ \tilde{P}_{s t}(\myvec{A}) \cdot K_{s t}^2 $ takes all quadratic terms of the matrix entries of $ \myvec{K} $ in the surrogate functional into account.\newline
The same can be done with the linear terms $ K_{st}, $ which leads to the coefficient
\begin{equation*}% \label{eq:prf:Separabilität von QTV:ZWelle}
	\tilde{Z}_{s t}(\vec{A}) \coloneqq -\psi_t \left ( \dfrac{1}{\vert \nabla_{st} A \vert} \sum_{r\in {N}_s} \left [ A_{s t} + A_{r t} \right ] + \sum_{r\in \bar{N}_s} \dfrac{A_{s t} + A_{r t}}{\vert \nabla_{rt} A \vert} \right ).
\end{equation*}
Therefore, the surrogate $ Q_{\text{TV}} $ can be written as
\begin{equation*}
	Q_{\TV}(\vec{K}, \vec{A}) = \sum_{t=1}^p \sum_{s=1}^n  \left [\tilde{P}_{s t}(\vec{A}) \cdot K_{s t}^2 + \tilde{Z}_{s t}(\vec{A}) \cdot K_{s t} \right ] \ + \tilde{C}(\vec{A})
\end{equation*}
for some function $ \tilde{C}, $ which only depends on $ \myvec{A}.$ This shows the separability of the surrogate.

\end{appendices}

\end{document}